\documentclass[11pt]{article}
\usepackage{style}

\usepackage{amsthm ,amsmath, amssymb, natbib, graphicx, url, algorithm2e}

\usepackage{times}

\usepackage[utf8]{inputenc}
\usepackage{amssymb}
\usepackage{amsfonts,amsmath, amsthm}
\usepackage{mathtools,braket}
\usepackage{enumerate}
\usepackage{multirow, multicol}
\usepackage{thmtools, thm-restate}
\usepackage{algpseudocode}
\usepackage{subcaption}
\usepackage{soul}
\usepackage{varwidth}
\usepackage{url}

\usepackage{soul} %
\usepackage{mathtools}
\usepackage{mdframed}
\usepackage{booktabs}
\usepackage{footmisc}

\newtheorem*{theorem*}{Theorem}

\newtheorem{theorem}{Theorem}
\newtheorem{lemma}[theorem]{Lemma}
\newtheorem{proposition}[theorem]{Proposition}

\newtheorem{definition}{Definition}

\usepackage{soul, cleveref}

\newcommand{\expectation}{\mathbb{E}}
\newcommand{\reals}{\mathbb{R}}
\newcommand{\prob}{\mathbb{P}}

\newcommand{\normal}{\mathcal{N}}

\newcommand{\hth}{\hat{\theta}}
\newcommand{\tth}{\theta^{*}}
\newcommand{\bX}{X}

\newcommand{\eps}{\gamma}
\newcommand{\ty}{\Tilde{y}}
\newcommand{\tby}{\Tilde{y}}

\newcommand{\Iden}{\mathcal{I}}
\newcommand{\kmeansopt}{OPT}
\newcommand{\tmu}{\Tilde{\mu}}

\newcommand{\sushant}[1]{}
\newcommand{\appex}[1]{}

\usepackage{amsmath,amsfonts,bm}

\title{Aggregating Data for Optimal and Private Learning}

\author{
Sushant Agarwal${^{\circ}}^*$ \enspace
Yukti Makhija$^{\dagger}$ \enspace
Rishi Saket$^{\dagger}$ \enspace
Aravindan Raghuveer$^{\dagger}$  \vspace{1mm} \\
$^{\dagger}$ Google DeepMind \ \ $^{\circ}$ Northeastern University\\
\texttt{agarwal.sus@northeastern.edu} \\
\texttt{\{yuktimakhija,rishisaket,araghuveer\}@google.com}
} 
\date{}
\begin{document}
\maketitle
\def\thefootnote{*}\footnotetext{work done during an internship at Google DeepMind.}

\begin{abstract}
  Multiple Instance Regression (MIR) and Learning from Label Proportions (LLP) are learning frameworks arising in many applications, where the training data is partitioned into disjoint sets or \emph{bags}, and only an aggregate label i.e., \emph{bag-label} for each bag is available to the learner. In the case of MIR, the bag-label is the label of an undisclosed instance from the bag, while in LLP, the bag-label is the mean of the bag's labels. In this paper, we study for various loss functions in MIR and LLP, what is the optimal way to partition the dataset into bags such that the utility for downstream tasks like linear regression is maximized. We theoretically provide utility guarantees, and show that in each case, the optimal bagging strategy (approximately) reduces to finding an optimal clustering of the feature vectors or the labels with respect to natural objectives such as $k$-means. We also show that our bagging mechanisms can be made \emph{label-differentially private}, incurring an additional utility error. We then generalize our results to the setting of Generalized Linear Models (GLMs). Finally, we experimentally validate our theoretical results.
\end{abstract}

\section{Introduction}
In traditional supervised learning, the training dataset is a collection of labeled \emph{instances} of the form $(\mathbf{x}, y)$, where $\mathbf{x} \in \reals^d$ is an instance or feature-vector with label $y$. In many applications however, due to lack of instrumentation or annotators~\citep{CHR,DNRS}, or privacy constraints~\citep{R10}, instance-wise labels may not be available. Instead, the dataset is partitioned into disjoint sets or \emph{bags}
of instances, and for each bag only one \emph{bag-label} is available to the learner. The bag-label is derived from the undisclosed instance-labels present in the bag
via some aggregation function depending on the scenario. The goal is to train a model predicting the labels of individual instances. We call this paradigm as learning from aggregate labels, which directly generalizes traditional supervised learning, the latter being the special case of unit-sized bags. The two formalizations of our focus are (i) multiple instance regression (MIR) where the bag-label is one of the instance-labels of the bag, and the instance whose label is chosen as the bag-label is not revealed, and (ii) learning from label proportions (LLP) in which the bag-label is the average of the bag's instance-labels. In MIR as well as in LLP, our work considers real-valued instance-labels with regression\sushant{linear?} as the underlying instance-level task.

Due to increasing concerns over data privacy, recent regulations on sharing user-level signals across platforms have resulted in aggregation of data, resulting in LLP and MIR formulations for predictive model training on revenue critical advertising datasets (e.g. Apple SKAN and Chrome Privacy Sandbox, see \cite{o2022challenges}). In many situations, the learner can be an untrusted party, and we wish to protect the privacy of individual instance labels from the learner (and any downstream observer of the learners output), while still allowing the learner to train useful models. We assume the existence of a trusted \emph{aggregator} that has access to all the data, including feature vectors and labels. The aggregator partitions the instances into bags, and along with the bags also releases aggregate labels of each bag (i.e., the bag-label) to the learner. If a bag is of large size, revealing only the aggregate bag-label provides a layer of privacy protection of the labels, while on the other hand, larger bags in the training data lead to a loss in the quality (utility) of the trained model. Apart from the inherent privacy that MIR and LLP offer, the aggregator can further perturb the labels to obtain formal privacy guarantees in the sense of \emph{label differential privacy}, a popular notion of privacy that measures and prevents the leakage of label information.

In many applications, obtaining labeled data is very costly, but unlabeled data is relatively easy to acquire. This is especially relevant as training data is getting increasingly complex, and skilled human annotators are required for data-labeling, leading to semi-supervised learning settings \citep{van2020survey}. In such situations, the paradigm of learning from aggregate labels, especially MIR, can be very useful. Given a large amount of unlabeled data, and a limited labeling budget (say $m$), one could partition the data into $m$ bags, and query an annotator for the label of one of the instances in each bag\sushant{LLP?}. This setting naturally lends itself to the MIR formulation that we study. We call this process of partitioning unlabeled data into bags as \emph{label-agnostic} bagging. One might also be interested in the bagging of labeled data, for eg., due to privacy concerns as discussed earlier, which we call \emph{label-dependent} bagging. 

For various loss functions in MIR and LLP, we consider the task of optimal bag construction for both the \emph{label-agnostic} and \emph{label-dependent} settings. More specifically, we study the following question; what is the optimal strategy for the aggregator to partition the data into bags, such that the utility of downstream tasks such as linear regression is maximized.\sushant{merge?}

\paragraph{Outline}  In Section \ref{results}, we formally define the problem, and state our main results. We start with the task of linear regression, and define utility to be the closeness of the trained model to the target model (in the realizable setting). In the MIR setting (for the case of instance-level loss, where each instance is assigned their bag label), we show that the optimal bagging strategy corresponds to finding an optimal $k$-means clustering over the labels. In the LLP setting (for the case of bag-level loss, between the bag-label and average prediction of the bag), we prove that the optimal bagging strategy is label-agnostic, and involves minimizing the condition number of the covariance of the centroids of each bag. For MIR we also consider aggregate-level loss (between the bag-label and prediction of the bag centroid). Here, the utility bound involves both the $k$-means objective of instance-MIR, and the condition number objective of bag-LLP\sushant{check}. In Section \ref{privacy}, we also quantify the additional loss in utility incurred due to differential-privacy guarantees, in each of the previous scenarios.

In the following Section \ref{proofs}, we provide an overview of the analysis for the instance-MIR utility bound, and an upper bound for the condition number objective (which is common to both bag-LLP and aggregate-MIR) based on a random bagging approach. The rest of the proofs are moved to Appendix \ref{appendix:proofs} \appex{todo}. We then study the proposed bagging mechanisms through extensive experimentation in Section \ref{experiments}, and show that $k$-means clustering over the instances is an effective label-agnostic bagging heuristic for each of the cases we study. We analyse trends obtained by varying various parameters such as the minimum bag size, number of bags, and privacy budget. The rest of the experiments can be found in Appendix \ref{appendix:experiments} \appex{todo}. In Appendix \ref{GLM}, we generalize the previous results to GLM's, which includes popular paradigms such as logistic regression. We now discuss some of the most relevant previous work, deferring more detailed discussion to Appendix \ref{pw} \appex{todo}.

\subsection{Related Work} Learning from label proportions (LLP), in which the bag-labels are the average of the labels within the bag, started with the work of \cite{FK05}. Multiple instance regression (MIR) was introduced in \cite{RP01}, where the bag-label is one of the (real-valued) labels within the bag (in contrast to LLP in which it is their average). Popular baseline techniques apply instance-level regression by assigning the bag-label to the average feature-vector in the bag, called aggregated-MIR, or assigning the bag-label to each feature-vector in the bag, known as instance-MIR \citep{WRHOV08,RC05}. Both the above problems, LLP and MIR, have gained renewed interest due to recent restrictions on user data on advertising platforms leading to aggregate conversion labels in reporting systems \citep{Skan,priv,o2022challenges}. With the goal of preserving the utility of models trained on the aggregate labels, model training techniques for either randomly sampled~\citep{busafekete2023easy} or curated bags~\citep{chen2023learning} have been proposed. 

The case of instance-level loss for LLP has been studied in \cite{javanmard2024priorboostadaptivealgorithmlearning}, where they show that the optimal bagging strategy reduces to finding the best $k$-means clustering of the labels, very similar to our instance-MIR objective. This is not very surprising, as LLP and MIR are closely related. Indeed, the expected label of each bag in the MIR setup is exactly the label of the bag in the LLP case. Our focus is on MIR, and in addition we analyse the popular bag-level loss \citep{ArdehalyC17} for LLP. They provide an adaptive label-agnostic bagging heuristic, which assumes access to an oracle that provides bag-labels in an online setting. Our work provides label-agnostic bagging algorithm in each case, without assuming access to an online oracle. We provide formal privacy guarantees for each of our methods. They also discuss privacy guarantees for their heuristic algorithm; however, their approach does not provide formal privacy guarantees for label-dependent bagging, which we circumvent by using a private clustering algorithm.\sushant{expand}

\subsection{Our Results}\label{results}

The training dataset consists of $n$ samples $\in \reals^d$ denoted by $n \times d$ matrix $X$, of rank $d$, with the corresponding labels denoted by $n \times 1$ matrix $Y$. $X$ is partitioned into $m$\sushant{change} non-overlapping bags $B =\{B_1, \ldots, B_m\}$, each of size at least $k$, for some fixed $k$\footnote{We do not use $k$ as the number of clusters or bags, as is common in the use of $k$-means clustering.} (Hence, $n \geq mk$\sushant{assume n is a multiple of k?}). We consider the task of linear regression, and adopt a standard way to model it, where label $y_i = \bx_i^T \tth + \eps_i\,, \eps_i\sim\normal(0,\sigma^2)$, for a fixed underlying model $\tth$. We denote the expected value of the label of $x_i$ by $\ty_i$, i.e.,  $\ty_i:= \bx_i^T \tth$. An aggregator partitions $X$ into bags, and along with the feature-vectors in each bag also releases aggregate bag-labels of each bag to the learner. The learner's task is to find an estimator $\hth$, that is as close as possible to the underlying $\tth$. The problem of bag construction is for the aggregator to find an optimal bagging configuration such that a given loss function is minimized, while satisfying the minimum bag size constraint $|B_l| \geq k, \forall l \in [m]$. $B_l$ denotes the set of samples in bag $l$, and $\overline{y}_l$ denotes the aggregate response in bag $l$. In the case of MIR, we consider the popular case where the aggregate $\overline{y}_l$ is a uniformly random label, and for LLP $\overline{y}_l$ is the mean of the labels. Note that this minimum size constraint for the bags is essential to define a meaningful problem, otherwise the optimal bagging would be the trivial strategy of putting each point in a separate bag.

\sushant{proof techniques}

\subsubsection{MIR, Instance-level loss}\label{intro:MIR-instance}
\begin{definition}[Instance-level loss]
An estimator $\hth$ minimizes instance-level loss, if
\begin{equation}
    \label{eq:instanceloss}
        \hth
        :=
        \argmin_{\theta} \frac{1}{n} \sum_{l=1}^m \sum_{i\in B_l} \ell(\overline{y}_l,f_{\theta}(x_i))\,,
\end{equation}
where $\ell$ is the squared loss.
\end{definition}
In the case of instance-level loss, we basically assign the aggregate label of the bag to each point in the bag. The result below provides an upper bound on the utility. All expectations henceforth are conditioned on a fixed $X$, unless otherwise stated.

\begin{theorem}\label{thm:MIR-event}
For $\hth$ as in \eqref{eq:instanceloss}, for a given bagging $B$,
\begin{align}\label{eq:MIR-event-UB1}
    \E\left[ ||\hth-\tth||_2^2 \right] \le 
    \ C_1
      \left ( C_2 - \sum_{\ell=1}^m \frac{\left(\sum_{i\in B_\ell} \ty_i\right)^2}{|B_\ell|} \right ),
\end{align}

where constants $C_1, C_2$ are independent of $B$.
\end{theorem}
In Appendix \ref{appendix:proofs} we show that finding the optimal $k$-means clustering of the (expected) labels $\ty$ exactly minimizes \sushant{typo} $\sum_{\ell=1}^m \frac{\left(\sum_{i\in B_\ell} \ty_i\right)^2}{|B_\ell|}$. Hence, minimizing \eqref{eq:MIR-event-UB1} over the set of all baggings
amounts to the following optimization problem.
\begin{align}\label{scaled_kmeans}
    &\min_{B \in \mathcal{B}} \quad \sum_{l=1}^m \sum_{\ty_i\in B_l} (\ty_i - \mu_l)^2 \noindent \nonumber\\
  & ~\text{  subject to} \quad~~ |B_l| =  k \quad \forall l \in [m] 
\end{align}
where $\mu_l$ is the mean of the labels in $B_l$, and $\mathcal{B}$ denotes the set of all baggings of the~$n$ samples. Note that the optimization involves use of $\ty$ which is unavailable, but one can instead use $y$ as a proxy, leading to a small additional utility error of $\left(1-\frac{1}{k}\right)\sigma^2$ (Appendix \ref{appendix:proofs})\sushant{check}. The $1d$ clustering problem above can be solved exactly in polynomial time, and turns out to result in a bagging that just sorts the labels in order, and partitions contiguous segments into bags (Appendix \ref{appendix:proofs}). In Appendix \ref{appendix:k_means}, we also justify that $k$-means clustering of the instances $X$ is a good proxy for the $k$-means clustering of the labels $y$, leading to a label-agnostic bagging.

\subsubsection{LLP, Bag-level loss}
\begin{definition}[Bag-level loss]An estimator $\hth$ minimizes bag-level loss, if
\begin{equation}
    \label{eq:bagloss}
        \hth
        :=
        \argmin_{\theta} \frac{1}{m} \sum_{l=1}^m  \ell\left(\overline{y}_l, \frac{\sum_{i\in B_l} f_{\theta}(x_i)}{|B_l|}\right)\,.
\end{equation}
\end{definition}
The loss is between the bag-label and mean of the instance level predictions of the bag instances. Below, we provide an upper bound on the utility for equal sized bags (we also show a corresponding result without the equality constraint in Appendix \ref{appendix:proofs}\appex{todo}).

\begin{theorem}\label{thm:LLP-bag-UB1}
For $\hth$ as in \eqref{eq:bagloss}, for a given bagging $B$ $~\text{such that }  |B_l| = k, \forall l \in [m]$,
\begin{align}\label{eq:LLP-bag-UB1}
    \E\left[ \|\hth-\tth\|_2^2  \right] \leq 
    \sigma^2 \frac{m}{k} \left(\frac{\lambda_{max}(f(X))}{\lambda_{min}(f(X))}\right)^2, 
\end{align}
where $\lambda_{max}/\lambda_{min}$ denote the maximum/minimum eigenvalues of a matrix, and $f(X) = g(X)g(X)^T$, for
$    g(X) = \left[\left(\frac{\sum_{i\in B_1} x_i}{|B_1|}\right), \dots, \left(\frac{\sum_{i\in B_m} x_i}{|B_m|}\right)\right].$ 

\end{theorem}

Essentially, $f(X)$ is the (sample) covariance matrix of each bag-centroid. The optimal bagging strategy involves minimizing the condition number (ratio of the maximum and minimum eigenvalue) of $f(X)$. In Appendix \ref{appendix:k_means}, we justify that finding an optimal $k$-means clustering of the instances $X$ is a good proxy for minimizing the condition number. In addition, in Section \ref{random-bagging}, we show that even random bagging gives us a reasonable upper bound. Note that the optimal bagging strategy here does not involve knowledge of the labels, leading to equally good utility for label-agnostic and label-dependent bagging.

\subsubsection{MIR, Aggregate-level loss}

\begin{definition}[Aggregate-level loss]
An estimator $\hth$ minimizes aggregate-level loss, if
\begin{equation}
    \label{eq:aggloss}
        \hth
        :=
        \argmin_{\theta} \frac{1}{m} \sum_{l=1}^m  \ell\left(\overline{y}_l,f_{\theta}\left(\frac{\sum_{i\in B_l} x_i}{|B_l|}\right)\right)\,
\end{equation}
\end{definition}
The loss is between the bag-label and prediction of the centroid of the bag instances. Below, we provide an upper bound on the utility for equal sized bags (we also show a corresponding result without the equality constraint in Appendix \ref{appendix:proofs}.

\begin{theorem}\label{thm:MIR-agg-UB1}
For $\hth$ in \eqref{eq:aggloss}, given a bagging $B$ $~\text{such that }  |B_l| = k, \forall l \in [m]$, 
\begin{align}\label{eq:MIR-agg-UB1}
    &\E\left[ \|\hth-\tth\|_2^2  \right] \leq 
    C_1\left(\frac{\lambda_{max}(f(X))}{\lambda_{min}(f(X))}\right)^2 \left(C_2  +  \sum_{l=1}^m \sum_{\ty_i\in B_l} (\ty_i - \mu_l)^2 \right),
\end{align}
where constants $C_1, C_2$ are independent of $B$.

\end{theorem}

As in the case of bag-LLP, minimizing the first term in \eqref{eq:MIR-agg-UB1} corresponds to minimizing the condition number of $f(X)$, and minimizing the second term corresponds to finding the optimal $k$-means clustering of $\ty$. In Appendix \ref{appendix:k_means}, we justify that finding an optimal $k$-means clustering of the instances $X$ is an effective proxy for minimizing both the terms, providing a label-agnostic bagging. In Appendix \ref{appendix:random}\appex{todo}, we give an label-dependent bagging method which combines $k$-means over the labels, followed by random bagging step, that is also effective\sushant{todo}.

\subsubsection{Privacy}\label{privacy} In each of the previous scenarios, the aggregator can modify the bagging procedure to obtain formal label-differential privacy guarantees \citep{chaudhuri2011sample}, defined below.

\begin{definition}[Label DP]
A randomized algorithm $A$ taking a dataset as an input is $(\epsilon, \delta)$-label-DP if
for two datasets $D$ and $D'$ which differ only on the label of one instance, for any subset $S$ of outputs
of $A$, 
\begin{equation*}
    \prob[A(D) \in S] \leq e^{\epsilon} \prob[A(D') \in S] + \delta.
\end{equation*}
\end{definition}

To guarantee label-DP, it is necessary to assume a sensitivity bound on labels, which we achieve by bounding the norm of the labels by a constant $R$. In the results below, we quantify the additional loss in utility that is incurred due to private bagging, for instance-MIR and bag-LLP. We discuss the corresponding result for  aggregate-MIR in Appendix \ref{appendix:proofs}, along with the proofs.

\paragraph{MIR, Instance-level loss} 
    \begin{theorem}\label{thm:instance-mir-loss-priv}
There exists a bagging $B~\text{with }  |B_l| = k, \forall l \in [m]$, satisfying $(\epsilon,\delta)$ label-DP, such that for $\hth$ in \eqref{eq:instanceloss}, we have
\begin{align}
    &\E\left[ ||\hth-\tth||_2^2 \right]
    \leq C_1\left(C_2 + \kmeansopt + n\left(1-\frac{1}{k}\right)\alpha^2  + \frac{d\alpha^2}{k^2} \right), \nonumber
\end{align}
where $\alpha^2 = \frac{16R^2\log\left(\frac{1.25}{\delta/2}\right) }{\epsilon^2}$, $OPT$ is the objective value of the optimal $k$-means clustering over $\ty$, and constants $C_1, C_2$ are independent of $B$.
\end{theorem}
In the label-agnostic setting, one would just need to add noise to the bag-labels. MIR outputs one label at random, hence the sensitivity of the output is $2R$. Due to privacy amplification via subsampling \cite{balle2018privacyamplificationsubsamplingtight}, we add $\normal\left(0,\frac{\alpha^2}{k^2} \right)$ noise to the label value to ensure $(\frac{\epsilon}{2},\frac{\delta}{2})$ label-DP, where $\alpha^2 = \frac{16R^2\log\left(\frac{1.25}{\delta/2}\right) }{\epsilon^2}$, leading to an additional error of $\frac{d\alpha^2}{k^2}$. In addition, since the objective here is a label-dependent clustering, we must use a differentially private $k$-means algorithm, leading to additional loss in utility. We show that the simple approach of adding $\normal\left(0,\alpha^2 \right)$ noise to each label, and then find an optimal clustering over the noise labels, leads to an additional error of $n\left(1-\frac{1}{k}\right) \alpha^2$. In Appendix \ref{appendix:proofs}\appex{todo}\sushant{separate appendix?}, we discuss how it is possible to achieve better utility, since the above method satisfies the more stringent notion of local-DP\sushant{cite}, while we only need to satisfy the standard notion of central-DP.

\paragraph{LLP, Bag-level loss}
\begin{theorem}\label{thm:bag-llp-loss-priv}
There exists a bagging $B~\text{with }  |B_l| = k, \forall l \in [m]$, satisfying $(\epsilon,\delta)$ label-DP, such that for $\hth$ in \eqref{eq:bagloss}, we have
\begin{align}
    &\E\left[ \|\hth-\tth\|_2^2 \right]
    = OPT \left( \frac{d}{k}\alpha^2 + \sigma^2 \frac{m}{k} \right),\nonumber
\end{align}
where $\alpha^2 = \frac{4R^2\log\left(\frac{1.25}{\delta}\right) }{\epsilon^2}$, and $OPT$ is the optimal value of $\left(\frac{\lambda_{max}(f(X))}{\lambda_{min}(f(X))}\right)^2$.
\end{theorem}
In this case, the optimal bagging strategy in independent of the labels. Hence, one just needs to add noise to the bag-labels, and not add noise for a private clustering of the labels. LLP outputs the mean of $k$ labels, hence the sensitivity of the output is $\frac{2R}{k}$. We add $\normal\left(0,\frac{\alpha^2}{k^2} \right)$ noise to the label value to ensure $(\epsilon,\delta)$ label-DP, leading to an additional error of $\frac{\alpha^2m}{k^2}$ over the corresponding non-private bagging mechanism.

\subsubsection{GLMs} Subsequently, we generalize the previous results for linear regression to the setting of Generalized Linear Model's (GLMs), which includes popular paradigms such as logistic regression. We study both instance-level and aggregate-level losses for MIR under the GLM framework. For instance-MIR, we derive an upper bound that leads to label k-means clustering as the optimal bagging strategy. This result holds across all distributions within the exponential family. For aggregate-MIR, our objective suggests minimizing the range between the maximum and minimum expected instance labels within a bag, implying that features with similar expected labels should be grouped together, yielding a clustering-based outcome. This result holds for exponential distributions which have a monotonic first derivative. The detailed analysis is provided in Appendix \ref{GLM} \appex{todo}.

\section{Utility analysis}\label{proofs}

\subsection{MIR, Instance-level loss} \label{MIR-event}

We denote the uniform distribution by $\Gamma$\sushant{$U$?}.
Let $\overline{y} = [\overline{y}_1, \dots, \overline{y}_m]$, where $\overline{y}_l = y_{\Gamma(B_l)}$. We define a random attribution matrix for MIR, $A$ $\in \{0, 1\}^{n \times n}$, as follows.
    \begin{align}\label{eq:Aevent}
        {A}_{(i,j)} =
        \begin{cases}
            1 & \text{if } i \in B_l \text{ and } \overline{y}_l = y_j\\
            0 & \text{otherwise}.
        \end{cases}
    \end{align}

Note that $\E[A] = S = S^T$ is given by
\begin{align}\label{eq:Svent}
    {S}_{(i,j)}  =
    \begin{cases}
        \frac{1}{|B_l|} & \text{if } i, j \in B_l \\
        0 & \text{otherwise}.
    \end{cases}
\end{align}
The minimizer of \eqref{eq:instanceloss} is then given by
\begin{align}\label{eq:MIR_event_theta}
    \hth =\argmin_{\theta} \frac{1}{n}\|Ay - X\theta\|_2^2 = (\bX^T\bX)^{-1}\bX^TA\by.
\end{align}

We now give a proof sketch for Theorem \ref{thm:MIR-event}, providing an upper bound for the error of $\hth$ (some details are omitted to Appendix \ref{proofs} \appex{todo}). All the expectations henceforth are over the randomness in $A$ unless otherwise stated.

\begin{proof}(of Theorem \ref{thm:MIR-event}) We begin with the following proposition, and use it to prove the main theorem
\begin{proposition}\label{prop:MIR-event-error}
\begin{align*}
    \E\left[ ||\hth-\tth||_2^2 \right]
    &=
  \E\left[   || (\bX^T\bX)^{-1}\bX^T(A - I)\bX\tth ||_2^2 \right]
    + \sigma^2 \E\left[ ||(\bX^T\bX)^{-1}\bX^TA||_{F}^2 \right].
    \end{align*} \label{thm:event-mir-loss-linear-regression}
\end{proposition}
\vspace{-40pt}
\begin{proof} (of Proposition \ref{prop:MIR-event-error})
    By rearranging the terms,
\begin{align}
\hth - \tth &= (\bX^T\bX)^{-1}\bX^TA\by - \tth \nonumber\\ &= (\bX^T\bX)^{-1}\bX^TA\bX\tth -\tth + (\bX^T\bX)^{-1}\bX^TA \eps \nonumber\\\nonumber
&= (\bX^T\bX)^{-1}\bX^T(A - I)\bX\tth+ (\bX^T\bX)^{-1}\bX^TA \eps\,.
\end{align}
$\eps$ is independent of $A$ with
$\E[\eps] = 0$, $\E[\eps\eps^T]=\sigma^2I$ and $\E[A] = S$. Using this we get, 
\begin{align*}
    \E\left[ ||\hth-\tth||^2  \right] &=
  \E\left[ || (\bX^T\bX)^{-1}\bX^T(A - I)\bX\tth ||_2^2  \right]  
    + \E\left[ {\rm tr} ((\bX^T\bX)^{-1}\bX^T A\eps\eps^TA^T \bX (\bX^T\bX)^{-1})  \right]\,\nonumber \\
      &=
  \E\left[   || (\bX^T\bX)^{-1}\bX^T(A - I)\bX\tth ||_2^2  \right]  
    + \sigma^2 \E\left[ {\rm tr} ((\bX^T\bX)^{-1}\bX^T AA^T \bX (\bX^T\bX)^{-1})  \right]\,\nonumber \\
    &= \E\left[   || (\bX^T\bX)^{-1}\bX^T(A - I)\bX\tth ||_2^2  \right] 
    + \sigma^2 \E\left[ ||(\bX^T\bX)^{-1}\bX^TA||_{F}^2  \right]\,\nonumber
\end{align*}
\vspace{-10pt}
\end{proof}
We now upper bound the error in Proposition \ref{prop:MIR-event-error}. We simplify the first term. 
\begin{align*}
    \E\left[||(\bX^T\bX)^{-1}\bX^T(A-I)\bX\tth||_2^2  \right]
    &\le
    \E\left[||(\bX^T\bX)^{-1}\bX^T||_{op} ||(A-I)\bX\tth||_2^2  \right]\\
    &= ||(\bX^T\bX)^{-1}\bX^T||_{op}^2\E\left[||(A-I)\bX\tth||_2^2\right]
\end{align*}%
We simplify the RHS above with the following proposition.
\begin{proposition}\label{event-MIR-k-means}
\begin{align*}
    \E \left[  ||(A - I)\bX\tth||_2^2 
\right] = \left(2||\Tilde{y}||_2^2 - 2\sum_{l=1}^m \frac{\left(\sum_{i\in B_l} \ty_i\right)^2}{|B_l|}
    \right)
    \end{align*}
\end{proposition}
\vspace{-20pt}
\begin{proof}
\begin{align*}
&\E \left[  ||(A - I)\bX\tth||_2^2 
\right] \\&= \E \left[ ((A - I)\bX\tth)^T  (A - I)\bX\tth 
\right] \nonumber\\
&= \E \left[ {\tth}^TX^TA^T  A \bX\tth 
\right] - \E \left[ {\tth}^TX^T (A + A^T)\bX\tth 
\right] + ||X\tth||_2^2 \\
&= \E \left[ ||A \Tilde{y}||_2^2 
\right] - {\tth}^TX^T (S + S^T)X\tth + ||X\tth||_2^2 \\
&= \E \left[ ||A \bX\tth||_2^2 
\right] - 2 {\tth}^TX^TSX\tth + ||\Tilde{y}||_2^2
\end{align*}
Putting the following two lemmas together, we conclude Proposition \ref{event-MIR-k-means}.
\begin{lemma}
\label{lem:exp_A_tilde_y}
    $\E \left[ ||A X\tth||_2^2 \right] = ||\Tilde{y}||_2^2$.
\end{lemma}
\vspace{-10pt}
\begin{proof}
    (of Lemma \ref{lem:exp_A_tilde_y})
  Let $B(i)$ be the bag containing $x_i$.  Note that $ A\bX\tth = \left[\ty_{\Gamma(B(1))}, \ldots, \ty_{\Gamma(B(n))}\right]^T$
\begin{align*}
    {\tth}^TX^TA^T  A \bX\tth &= \sum_{i=1}^{i=n} \ty^2_{\Gamma(B(i))}
\end{align*}
Then we have 
\begin{align*}
\E \left[\sum_{i=1}^{i=n} \ty^2_{\Gamma(B(i))}\right]
&=
\sum_{i=1}^{i=n} \left(\sum_{j \in B(i)} \frac{\left(\ty_j\right)^2}{|(B(i))|}  \right)\\
&=
\sum_{l=1}^{l=m} |B_l| \left(\sum_{j \in B(i)}\frac{\left(\ty_j\right)^2}{|B_l|}  \right)\\
&=
\sum_{i=1}^{n}{\left(\ty_i\right)^2} \\
\end{align*}
\vspace{-20pt}
\end{proof}
\begin{lemma}
\label{lem:clustering_term_mir}
   $ {\tth}^TX^TSX\tth =  \sum_{l=1}^m \frac{\left(\sum_{i\in B_l} \ty_i\right)^2}{|B_l|}$.
\end{lemma}
\begin{proof} (of Lemma \ref{lem:clustering_term_mir}).
Note that $S = M^TM$, where $M \in \mathbb{R}^{m \times n}$ is defined as:
\begin{align*}
        {M}_{(i,j)} =
        \begin{cases}
            1/\sqrt{|B_i|} & \text{if } x_j \in B_i \\
            0 & \text{otherwise}.
        \end{cases}
\end{align*}
Thus, ${\tth}^TX^TSX\tth = {\tth}^TX^TM^TMX\tth = ||M\Tilde{y}||_2^2$.
    \begin{align*}
        ||M\Tilde{y}||_2^2
        &= \sum_{l=1}^{m} \left(\sum_{x_i \in B_l}\frac{1}{\sqrt{|B_l|}} \Tilde{y_i}\right)^2\\
        &= \sum_{l=1}^{m} \frac{1}{|B_l|}\left(\sum_{x_i \in B_l} \Tilde{y_i}\right)^2
    \end{align*}
    \vspace{-40pt}
\end{proof}
\end{proof}
The following proposition analyses the second term in Proposition \ref{prop:MIR-event-error}, and together with Proposition \ref{event-MIR-k-means} concludes the proof of Theorem \ref{thm:MIR-event}.\sushant{can we make these clickable?}
\begin{proposition}
\label{lem:exp_var_term_mir_event}
\begin{align*}
     \E\left[ ||(\bX^T\bX)^{-1}\bX^TA||_{F}^2 \right] \leq d ||(\bX^T\bX)^{-1}\bX^T||_{op}^2
\end{align*}
    
\end{proposition}
\vspace{-20pt}
\end{proof}

\subsection{LLP, Bag-level loss}\label{LLP-bag}
 We define a bagging matrix $S$ $\in \{0, 1\}^{m \times n}$ that encodes the assignment of instances to bags.
 \begin{align}\label{eq:S_LLP}
    S_{(l,i)} =
    \begin{cases}
        \frac{1}{|B_l|} & \text{if } i \in B_l,\\
        0 & \text{otherwise}.
    \end{cases}
\end{align}

The minimizer of the bag-level loss in matrix form is
\begin{align}\label{eq:bag_loss_matrix}
    \hth =\argmin_{\theta} \frac{1}{m}\|S \by - S\bX\theta\|_2^2.
\end{align}

Theorem \ref{thm:LLP-bag-UB1} provides an upper bound on the error for equal sized bags, showing that
\begin{align*}
    \E\left[ \|\hth-\tth\|_2^2  \right] \leq 
    \sigma^2 \frac{m}{k} \left(\frac{\lambda_{max}((S\bX)^T S\bX)}{\lambda_{min}((S\bX)^T S\bX)}\right)^2 .
\end{align*}
We want to develop a bagging algorithm that minimizes the condition number of the covariance of the bag-centroids. Since bounding the condition number as a whole is challenging, we instead find an upper bound for $\lambda_{max}$ (Lemma \ref{lem:lambda_max_bound}) and a lower bound $\lambda_{min}$ of $(S\bX)^T S\bX$. Aggregating feature vectors reduces the eigenvalues of the covariance matrix. We propose the following random bagging algorithm (Algorithm \ref{algo:random_bag_eigenvalue}), which provides a lower bound for the $\lambda_{min}((S\bX)^T S\bX)$.

\subsubsection{A Random bagging approach}\label{random-bagging}
Our bagging algorithm considers a fixed random partitioning strategy where the instances are divided into \textit{super}-bags, each containing 2k instances. From each super bag, one $k$-sized bag is independently sampled, resulting in a collection of $m/2$ bags. We then analyze the minimum eigenvalue of the covariance matrix for this subset of bags. Since covariance matrices are PSD, the minimum eigenvalue of this subset of bags is a lower bound for any collection of $m$ bags formed from the same instances, as adding more covariances will not reduce the minimum eigenvalue.
\begin{lemma}
\label{lem:lambda_max_bound}
$\lambda_{max}\left((S\bX)^TS\bX\right) \leq \lambda_{max}(X^TX)$.
\end{lemma}
Let $X_l$ represent the feature matrices of $B_l$ for $l \in [m]$.
\begin{equation*}
    \lambda_{min}\left((S\bX)^TS\bX\right) =  \frac{1}{k^2} \lambda_{min}\left(\sum_{l=1}^{m}X_l^TX_l\right) 
\end{equation*}

\begin{figure}[!htb]
\begin{mdframed}
\small
\textbf{Input:} : Instances $\mathcal{X}$, fixed bag size $k$. \\
\textbf{Steps:}
\begin{enumerate}
    \item Randomly partition $\mathcal{X}$ into $m'$ $2k$-sized \emph{super}-bags, where $m' = n/2k.$
    \begin{align*}
    \mathcal{X} = \bigcup_{l = 1}^{m'} \mathcal{X}_l \text{  and  } \mathcal{X}_l \bigcap \mathcal{X}_{l'} = \phi \text{  for all } l \neq l'%
    \end{align*}
    \item For $l = 1, \dots, m'$, a $k$-sized bag $B'_l$ is sampled  $u.a.r$ from $\mathcal{X}_l$.
    \item Output $\mathcal{B'}$ where $\mathcal{B'} = \{ B'_l \}_{l \in [m']}$
\end{enumerate}
\end{mdframed}
\caption{Random bagging algorithm for bag-LLP}\label{algo:random_bag_eigenvalue}
\end{figure}
The feature matrix for bag $B'_l$ sampled using Algorithm \ref{algo:random_bag_eigenvalue} can be represented by $X'_l$ for all $l \in [m']$. 
\begin{equation}
    \frac{1}{k^2} \lambda_{min}\left(\sum_{l=1}^{m}X_l^TX_l\right)  \geq \frac{1}{k^2} \lambda_{min}\left(\sum_{l=1}^{m}{X'_l}^TX'_l\right)
    \label{eq:bound_covariance_x'}
\end{equation}

Let $\mu_{min} = \lambda_{min}\left(\sum_{l=1}^{m'}\E\left[{X'_l}^TX'_l\right]\right)/k^2$. We expand ${X'_l}^TX'_l$ and find $\mu_{min}$:
\begin{align*}
    \mu_{min} &= \frac{1}{k^2} \lambda_{min} \left(\sum_{l=1}^{m'}\E\left[\sum_{x_i, x_j \in B'_l} x_ix_j^T\right]\right)\\
    &= \frac{1}{k^2} \lambda_{min} \left(\sum_{l=1}^{m'}\E\left[\sum_{x_i \in B'_l}x_ix_i^T\right] + \E\left[\sum_{i \neq j} x_ix_j^T\right]\right)
\end{align*}
\newcommand*{\Comb}[2]{{}^{#1}C_{#2}}%
In Algorithm \ref{algo:random_bag_eigenvalue}, $x_i \in \mathcal{X}_l$ get sampled in $B'_l$ with probability $1/2$. Similarly, the probability of sampling the ordered pair $(x_i, x_j)$ is $2\Comb{2k-2}{k-2}/\Comb{2k}{k} = (k-1)/(2k-1)$. Let $\Hat{x} = \sum_{x_i \in \mathcal{X}_l} x_i.$
\begin{align*}
    &\mu_{min} \\
     &= \frac{\lambda_{min}}{k^2}  \left(\sum_{l=1}^{m'}\sum_{x_i \in \mathcal{X}_l}\frac{1}{2}x_ix_i^T + \sum_{(x_i, x_j) \in \mathcal{X}_l}\frac{k-1}{2k-1} x_ix_j^T\right) \\
    &= \frac{\lambda_{min}}{k^2} \left(\sum_{l=1}^{m'}\frac{1}{2} \left(1 - \frac{k-1}{2k-1}\right)\sum_{x_i \in \mathcal{X}_l}x_ix_i^T + \frac{k-1}{2(2k-1)} \Hat{x}\Hat{x}^T \right)\\
    &= \frac{\lambda_{min}}{k^2} \left(\sum_{l=1}^{m'} \left(\frac{k}{2(2k-1)}\right)\sum_{x_i \in \mathcal{X}_l}x_ix_i^T + \frac{k-1}{2(2k-1)} \Hat{x}\Hat{x}^T \right)\\
    &= \frac{\lambda_{min}}{2k^2(2k-1)}  \left(k X^TX + (k-1)\sum_{l=1}^{m'} \Hat{x}\Hat{x}^T \right)
\end{align*}
Since the second term is a summation of $p.s.d$ matrices, we get $\mu_{min} > \lambda_{min}(X^TX)/4k^2$. We assume $\|x\|_2^2 \leq \beta$ for all $x \in \mathcal{X}$.
\begin{lemma}
$\lambda_{max}({X'_l}^TX'_l) \leq k\beta.$
\end{lemma}
Applying Matrix Chernoff (Corollary 5.2 \cite{tropp2012user}), we get
\begin{align*}
\prob \left[ \frac{1}{k^2}\lambda_{\min}\left(\sum_{l=1}^{m} {X'_l}^TX'_l \right) \leq (1 - \delta) \mu_{\min} \right] 
	\leq  d \cdot \left[ \frac{e^{-\delta}}{(1 - \delta)^{1-\delta}} \right]^{\mu_{\min}/k\beta} 
\end{align*}
Using Equation \ref{eq:bound_covariance_x'} we get
\begin{align*}
\prob \left[ \lambda_{min}\left((S\bX)^TS\bX\right) > (1 - \delta) \frac{\lambda_{min}(X^TX)}{4k^2} \right] 
	\geq 1 - d \cdot \left[ \frac{e^{-\delta}}{(1 - \delta)^{1-\delta}} \right]^{\mu_{\min}/k\beta} 
\label{eq:min_eigenvalue_bound_prob}
\end{align*}

Using Lemma \ref{lem:lambda_max_bound} and Equation \eqref{eq:min_eigenvalue_bound_prob}, we get
\begin{align}
     \E\left[ \|\hth-\tth\|_2^2  \right] \leq 
    \frac{16\sigma^2nk^2}{(1-\delta)^2} \left( \frac{\lambda_{max}(X^TX)}{\lambda_{min}(X^TX)}\right)^2.
\end{align}
$w.p.$ greater than $1 - d \cdot \left[ \frac{e^{-\delta}}{(1 - \delta)^{1-\delta}} \right]^{\mu_{\min}/k\beta}$.

\subsection{MIR, Aggregate-level loss}\label{MIR-agg}

\sushant{remove, or merge}
We define a random attribution matrix $A$ $\in \{0, 1\}^{m \times n}$ as follows, to indicate the bag-label of each bag.
\begin{align}\label{eq:A_MIR}
    A_{(l,i)} =
    \begin{cases}
        1 & \text{if } y_i = \Gamma{(B_l)},\\
        0 & \text{otherwise}.
    \end{cases}
\end{align}

We denote $\E[A] = S$. This turns out to be the same S as \eqref{eq:S_LLP}, and represents the instances in each bag. The minimizer of the bag-level loss is
\begin{align}\label{eq:bag_loss_matrix}
    \hth =\argmin_{\theta} \frac{1}{m}\|A \by - S\bX\theta\|_2^2.
\end{align}
Theorem \ref{thm:MIR-agg-UB1} provides an upper bound on the error for equal sized bags, showing that
\begin{align*}
    \E\left[ \|\hth-\tth\|_2^2 \right]
    \leq C_1 \left(\frac{\lambda_{max}((S\bX)^T S\bX)}{\lambda_{min}((S\bX)^T S\bX)}\right)^2 \nonumber \left( C_2 + \sum_{l=1}^m \sum_{\ty_i\in B_l} (\ty_i - \mu_l)^2 \right).
\end{align*}

\section{Experiments}\label{experiments}

\begin{table}[tbh]
\begin{minipage}{0.49\linewidth}
\small{
    \centering
    \begin{tabular}{rrrrrr}
\toprule
 $k$ & Bagging Method & $\|\hth-\tth\|_2^2$ \\
\midrule
\textit{LLP} &\textit{Bag Loss} \\
\midrule
\multirow[c]{3}{*}{10} & Instance $k$-means & $0.0082 \pm 0.002$ \\
& Label $k$-means & $0.0458 \pm 0.012$ \\
& Random & $0.0099 \pm 0.002$ \\
\cline{2-3}
\multirow[c]{3}{*}{50} & Instance $k$-means & $0.0392 \pm 0.008$ \\
& Label $k$-means & $0.0629 \pm 0.008$ \\
& Random & $0.0423 \pm 0.009$ \\
\midrule

\textit{MIR} & \textit{Instance Loss}\\
\midrule
\multirow[c]{3}{*}{10} & Instance $k$-means & $0.0088 \pm 0.002$ \\
 & Label $k$-means & $0.0072 \pm 0.002$ \\
 & Random & $0.0085 \pm 0.002$ \\
\cline{2-3}
\multirow[c]{3}{*}{50} & Instance $k$-means & $0.0388 \pm 0.006$ \\
 & Label $k$-means & $0.0404 \pm 0.007$ \\
 & Random & $0.0419 \pm 0.006$ \\

\midrule
\textit{MIR} & \textit{Aggregate Loss}\\
\midrule
\multirow[c]{3}{*}{10} & Instance $k$-means & $0.0102 \pm 0.002$ \\
 & Label $k$-means & $0.0453 \pm 0.008$ \\
 & Random & $0.0221 \pm 0.004$ \\
\cline{2-3}
\multirow[c]{3}{*}{50} & Instance $k$-means & $0.0437 \pm 0.008$ \\
 & Label $k$-means & $0.0601 \pm 0.008$ \\
 & Random & $0.0619 \pm 0.012$ \\
\bottomrule
\end{tabular}
    \caption{Non-Private Bagging}
    \label{tab:non_private_llp_bag_inst_mir}
}
\end{minipage}
\hfill
\begin{minipage}{0.49\linewidth}
\small{
    \centering
    \begin{tabular}{rrrrrr}
\toprule
 $k$ & Bagging Method & $\epsilon$ & $\|\hth-\tth\|_2^2$ \\
\midrule
\textit{MIR} & \textit{Instance Loss}\\
\midrule
\multirow[c]{9}{*}{10} & \multirow[c]{3}{*}{Instance $k$-means} & 0.5 & $0.0621 \pm 0.009$ \\
& & 1.0 & $0.0537 \pm 0.009$ \\
& & 2.0 & $0.0390 \pm 0.008$ \\
\cline{2-4}
& \multirow[c]{3}{*}{Label $k$-means} & 0.5 & $0.0505 \pm 0.005$ \\
& & 1.0 & $0.0362 \pm 0.006$ \\
& & 2.0 & $0.0189 \pm 0.004$ \\
\midrule
\multirow[c]{9}{*}{50} & \multirow[c]{3}{*}{Instance $k$-means} & 0.5 & $0.0656 \pm 0.012$ \\
& & 1.0 & $0.0595 \pm 0.012$ \\
& & 2.0 & $0.0521 \pm 0.009$ \\
\cline{2-4}
& \multirow[c]{3}{*}{Label $k$-means} & 0.5 & $0.0559 \pm 0.008$ \\
& & 1.0 & $0.0480 \pm 0.005$ \\
& & 2.0 & $0.0431 \pm 0.006$ \\
\bottomrule
\end{tabular}
    \caption{Private Bagging}
    \label{tab:private_llp_bag_inst_mir}
}
\end{minipage}
\end{table}

We conduct experiments on synthetically generated data. The synthetic dataset is of the form $(X \in \mathbb{R}^{n\times d}, y \in \mathbb{R}^n)$ and is generated by first sampling a random ground truth model $\theta^*$ from the standard $d$-dimensional Gaussian distribution, sampling each of the rows of $X$ iid from the standard\sushant{isotropic?} $d$-dimensional Gaussian distribution and then setting $y = X\theta^* + \eps$ where each coordinate of $\eps$ is iid drawn from $N(0, \sigma^2)$ where $\sigma$ is 0.5. We set $n$ to be $50,000$ and $d$ as $32$. We also vary $k$, and use $k = 10, 50$. 

We implement 3 bagging mechanisms on each of instance-MIR, aggregate-MIR, and bag-LLP, namely (1) Instance $k$-means\sushant{rename table}, (2) Label $k$-means, and (3) Random bagging\sushant{different from proof, mention?}. 
In Table \ref{tab:non_private_llp_bag_inst_mir}, we present the mean and standard deviation of the error, calculated over $15$ runs for each experiment\sushant{where does the randomness in label k-means for LLP come in?}. As expected, for bag-LLP, instance $k$-means performs better than random bagging, which in turn performs better than label $k$-means. For aggregate-MIR, instance $k$-means consistently performs the best, which is expected, while random bagging overall performs slightly better than label $k$-means. However, for instance-MIR, all the 3 mechanisms show similar performance. 

We also consider the private version of instance-MIR in Table \ref{tab:private_llp_bag_inst_mir}. We set $\delta = 10^{-5}$, and vary $\epsilon$. For each mechanism, we see that accuracy drops with a decrease in $\epsilon$. However, the drop is sharper for label $k$-means, which is expected, since unlike feature $k$-means, it is label-dependent, incurring an extra utility error. We also note that that drop in accuracy is sharper for a smaller bag size; this is again expected since the error due to privacy scales with $\frac{1}{k}$. We conduct more extensive experiments in Appendix \ref{appendix:experiments} \appex{todo}.

\subsection{Instance $k$-means}
We justify that $k$-means of the instances X is an effective label-agnostic bagging heuristic for each setting we consider, and provide more details in Appendix \ref{appendix:k_means} \appex{todo}.

\paragraph{Instance-MIR} Note that in our setting of linear regression, $\ty =\tth$. In other words, $\ty$ is just the projection of $X$ along the axis normal to the hyperplane determined by $\tth$. Hence, finding an optimal $k$-means clustering of $\ty$ is equivalent to minimizing the $k$-means objective of projections along this axis. However, if the labels are not given, this axis is unknown, since $\tth$ is unknown. Hence, in order to do a label-agnostic bagging, one must minimize some objective that simultaneously reduces the $k$-means objective along every direction. In Appendix \ref{appendix:k_means} \appex{todo}, we justify that $k$-means of the instances X is a good heuristic for the same.

\paragraph{Bag-LLP} We want to maximize $\lambda_{\min}((SX)^T SX),$ where $(SX)^T SX$ is the sample covariance matrix of the centroids of each bag. $\lambda_{\min}$ of a covariance matrix measures the variance along the corresponding eigenvector (which is also the direction of least variance). In Appendix \ref{appendix:k_means} \appex{todo}, we show that maximizing the variance of bag-centroids along a direction is equivalent to finding an optimal $k$-means on $X$ projected on that direction. In order to maximize $\lambda_{\min}((SX)^T SX),$ we maximize variance in every direction. Equivalently, we want to reduce the $k$-means objective along every direction. In Appendix \ref{appendix:k_means} \appex{todo}, we justify that $k$-means of the instances X is a good heuristic for the same. 

\paragraph{Aggregate-MIR} Note that in order to minimize the error bound, we must simultaneously minimize the condition number of $(SX)^T SX,$ and the $k$-means objective over the labels $\ty$. Earlier, we justified that $k$-means of the instances X is a good heuristic for both objectives.

\section{Conclusion}\label{conclusion}
In this paper, we study for various loss functions in MIR and LLP, what is the optimal way to partition the dataset into bags such that the utility for downstream tasks like linear regression is maximized. We theoretically provide utility guarantees, and show that in each case, the optimal bagging strategy (approximately) reduces to finding an optimal $k$-means clustering of the feature vectors or the labels. We also show that our bagging mechanisms can be made label-DP, incurring an additional utility error. We finally generalize our results to the setting of GLMs. 

There are several potential directions for future work. While we only considered linear models, it would be interesting to analyse optimal bagging strategies in non-linear models, such as neural networks. One could also consider other popular loss functions for MIR and LLP used in literature. While our work only looked at upper bounds, having corresponding lower bounds would also be interesting.\sushant{more?}

\bibliography{References}

\appendix 
\onecolumn
\section*{Aggregating Data for Optimal and Private Learning: \\
Supplementary Materials}
\paragraph{Outline} Appendix \ref{pw} includes a more detailed discussion of previous work. In Appendix \ref{appendix:proofs}, we present the missing proofs from the paper, along with some additional results that were briefly mentioned in the main paper. Additional experiments can be found in Appendix \ref{appendix:experiments}. Appendix \ref{appendix:k_means} justifies that $k$-means of the instances $X$ is an effective label-agnostic bagging heuristic for each setting we consider (instance-MIR, bag-LLP, and aggregate MIR). In Appendix \ref{appendix:random}, we discuss the super-bags algorithm for Agg-MIR which is a combination of label $k$-means and random bagging. In Appendix \ref{GLM}, we generalize previous results to GLM's.

\section{Related Work}\label{pw}

\noindent
{\bf Learning from aggregated labels.}
Learning from label proportions (LLP), in which the bag-labels are the average of the labels within the bag, started with the work of \cite{FK05} and has been studied in the context of privacy concerns~\citep{R10}, lack of supervision due to cost~\cite{CHR}, or coarse instrumentation~\cite{DNRS}. While previous works \citep{QSCL09,YLKJC13,KDFS15,LWQTS19,SZ20,SRR} have developed specialized techniques for model training on LLP training data, \cite{YCKJC14} defined it in the PAC framework, while \cite{Saket21,Saket22} have shown worst case algorithmic and hardness bounds, and recently \cite{brahmbhatt2023pac} gave PAC learning algorithms for Gaussian feature vectors and random bags. \\
A related formulation is that of multiple instance regression (MIR), introduced in \cite{RP01}, where the bag-label is one of the (real-valued) labels within the bag (in contrast to LLP in which it is their average). For the most part, MIR has been studied in applied settings related to remote sensing and image analysis. Popular baseline techniques apply instance-level regression by assigning the bag-label to the average feature-vector in the bag, called aggregated-MIR, or assigning the bag-label to each feature-vector in the bag, known as instance-MIR \citep{WRHOV08,RC05}, whereas several expectation-maximization (EM) based methods have also been proposed \citep{RP01,WRHOV08,WLV7,WLR08,TF18}. Recent work of \cite{KSABGR} proved bag-to-instance generalization error bounds as well as hardness results for MIR, in the first theoretical exploration of this problem.

\medskip
Both the above problems, LLP and MIR, have gained renewed interest due to recent restrictions on user data on advertising platforms leading to aggregate conversion labels in reporting systems \citep{Skan,priv,o2022challenges}. With the goal of preserving the utility of models trained on the aggregate labels, model training techniques for either randomly sampled~\citep{busafekete2023easy} or curated bags~\citep{chen2023learning} have been proposed. More recently, \cite{javanmard2024priorboostadaptivealgorithmlearning} showed that minimizing a natural instance-level loss for LLP yields the best utility when the bags are created by optimizing $k$-means objective defined over the constituent labels of the bags, for linear regression tasks. We note however, that such a treatment of the equally well used bag-loss method~\citep{ArdehalyC17} for LLP is lacking, while for MIR this topic of optimal bag creation has not been studied.

{\bf Label Differential Privacy}. 
Differential privacy (DP), by now a standard notion of privacy of algorithmic mechanisms, was introduced by \cite{dwork2006differential}. In the context of training datasets, the restricted notion of label-differential privacy (label-DP) was provided by \cite{chaudhuri2011sample}. Recent works have provided label-DP mechanisms for classification~\citep{ghazi2021deep} and regression~\citep{ghazi2022regression} while \citep{esfandiari2022label} proposed clustering based mechanisms.

\section{MISSING PROOFS}\label{appendix:proofs}

In this section, we present the missing proofs from the paper, along with some additional results that were briefly mentioned in the main paper. 

\subsection{Additional results from Section \ref{intro:MIR-instance}}

Lemma \ref{l1} shows that finding the optimal $k$-means clustering of the (expected) labels $\ty$ exactly maximizes $\sum_{\ell=1}^m \frac{\left(\sum_{i\in B_\ell} \ty_i\right)^2}{|B_\ell|}$. Lemma \ref{noisyclustering} shows that clustering over $y = \ty + \gamma$ as a proxy for clustering over $\ty$ leads to an additional utility error of $\left(1-\frac{1}{k}\right)\sigma^2n$. Lemma \ref{l3} shows that the $1d$ clustering problem above turns out to result in a bagging that just sorts the labels in order, and partitions contiguous segments into bags.

\begin{lemma}\label{l1}
Maximizing $\sum_{\ell=1}^m \frac{\left(\sum_{i\in B_\ell} \ty_i\right)^2}{|B_\ell|}$ corresponds to finding the optimal $k$-means clustering over $\ty$.
\end{lemma}
\begin{proof}
The $k$-means objective for a bagging $B$ over $\ty$ is 
\begin{align*}\label{eq:k_means}
     \sum_{l=1}^m \sum_{i\in B_l} (\ty_i - \mu_l)^2\,,
\end{align*}
where $\mu_{l} = \frac{1}{|B_l|} \sum_{i\in B_l} \ty_i$ is the mean of the entries of $\tby$ in bag $l$. We expand on the objective below. 
\begin{align*}
    \sum_{l=1}^m \sum_{i\in B_l} \left(\ty_i - \mu_l\right)^2 &= \sum_{l=1}^m \sum_{i\in B_l} (\ty_i^2 + \mu_l^2 - 2\ty_i\mu_l)\\
    &= \sum_{l=1}^m \left(\sum_{i\in B_l} \ty_i^2 + \sum_{i\in B_l} \mu_l^2 - 2\sum_{i\in B_l}\ty_i\mu_l\right)\\
    &= \sum_{l=1}^m \left(\sum_{i\in B_l} \ty_i^2 + |B_l| \mu_l^2 - 2|B_l| \mu_l^2\right)\\
    &= \sum_{i=1}^n \ty_i^2 - \sum_{l=1}^m \left( |B_l| \mu_l^2 \right)\\
    &= ||\ty||_2^2 - \sum_{\ell=1}^m \frac{\left(\sum_{i\in B_\ell} \ty_i\right)^2}{|B_\ell|}
\end{align*}
$||\ty||_2^2$ is constant, hence minimizing $\sum_{l=1}^m \sum_{i\in B_l} \left(\ty_i - \mu_l\right)^2$ is equivalent to maximizing $\sum_{\ell=1}^m \frac{\left(\sum_{i\in B_\ell} \ty_i\right)^2}{|B_\ell|}$.
\end{proof}
\begin{lemma}\label{noisyclustering}
Given $y_i = \ty_i + \gamma_i$, where $\eps_i\sim\normal(0,\sigma^2)$. Then, given a clustering $B$ over $y$,
\begin{align*}
    \expectation[k\text{-means}(B(y))] = \expectation[k\text{-means}(B(\ty))] + (n-m) \sigma^2
\end{align*}
where where $\text{k-means}(S(X))$ is the $k$-means clustering objective of $S$ on $X$. For equal sized bags of size $k$,
\begin{align*}
    \expectation[k\text{-means}(B(y))] = \expectation[k\text{-means}(B(\ty))] + n\left(1-\frac{1}{k}\right)\sigma^2 \text{\sushant{}}.
\end{align*}
\end{lemma}
\begin{proof}
    \begin{align*}
    \expectation[k\text{-means}(B(y))] -  \expectation[k\text{-means}(B(\ty))] &=
    \expectation \left[ \sum_{l=1}^m \sum_{i\in B_l} (y_i - \mu_l)^2 \right] - \expectation\left[\sum_{l=1}^m \sum_{i\in B_l} (\ty_i - \mu_l)^2 \right] \\ &=
       \expectation \left[ \sum_{l=1}^m \sum_{i\in B_l} (y_i - \mu_l)^2 - \sum_{l=1}^m \sum_{i\in B_l} (\ty_i - \mu_l)^2 \right] \\ &= \expectation\left[\sum_{l=1}^m \sum_{i\in B_l} \left((y_i - \mu_l)^2 -  (\ty_i - \tmu_l)^2\right)\right] \\
        &= \expectation \left[ \sum_{l=1}^m \sum_{i\in B_l} \left( (y_i - \ty_i + \tmu_l - \mu_l)  (y_i - \mu_l + \ty_i - \tmu_l ) \right)\right] \\
        &= \expectation \left[\sum_{l=1}^m \sum_{i\in B_l} \left( \left(\gamma_i -  \frac{\sum_{i\in B_l} \gamma_i}{|B_l|}\right)  \left(2y_i - 2\mu_l + \gamma_i -  \frac{\sum_{i\in B_l} \gamma_i}{|B_l|}\right) \right)\right]\\
       &= \sum_{l=1}^m \sum_{i\in B_l} \left( \expectation\left[ \gamma_i^2\right] +  \frac{\sum_{i\in B_l} \expectation\left[ \gamma_i^2\right]}{|B_l|^2} - 2\frac{ \expectation\left[ \gamma_i^2\right]}{|B_l|} \right)\\
       &= \sum_{l=1}^m \sum_{i\in B_l} \expectation\left[ \gamma_i^2\right] \left( 1 - \frac{1}{|B_l|} \right) \\
       &= \sigma^2 \sum_{l=1}^m \left(|B_l| - 1 \right) \\
       &= \sigma^2 \left(n - m \right)
    \end{align*}
\end{proof}

\begin{lemma}\label{l3}
Sort $\ty_i$ in non-increasing order as $\ty_{(1)}, \ldots, \ty_{(n)}$. There exists an optimal $k$-means clustering $B^*$ such that $\ty_{(i)}, \ty_{(j)} \in B^*_l \implies \ty_{(k)} \in B^*_l, \forall k \in \{i, i+1, \ldots, j\}$.
\end{lemma}
\begin{proof}
    Follows from Lemma 2.3 in \cite{javanmard2024priorboostadaptivealgorithmlearning}.
\end{proof}

\subsection{Additional Proofs from Section \ref{MIR-event}}
\label{app:event_mir_extra}

\begin{proof}(of Lemma \ref{lem:exp_var_term_mir_event}). 
We use the following inequality:
\[
    ||AB||_{F}^2 \le \min\left({||A||_{op}^2||B||_{F}^2, ||B||_{op}^2||A||_{F}^2}\right)\,.
\]
\begin{align*}
    \E\left[ ||(\bX^T\bX)^{-1}\bX^TA||_{F}^2  \right] \le
    \min\left(\E\left[ ||(\bX^T\bX)^{-1}\bX^T||_{op}^2||A||_{F}^2  \right], \E\left[ ||(\bX^T\bX)^{-1}\bX^T||_{F}^2||A||_{op}^2  \right]\right)
\end{align*}
We assumed $\text{rank}(\bX) = d$, hence
$||(\bX^T\bX)^{-1}\bX^T||_{F}\le \sqrt{d} ||(\bX^T\bX)^{-1}\bX^T||_{op}\,.$
\begin{align*}
 \E\left[ ||(\bX^T\bX)^{-1}\bX^TA||_{F}^2  \right] &\le \min\left(\E\left[ ||(\bX^T\bX)^{-1}\bX^T||_{op}^2||A||_{F}^2  \right], \E\left[ d||(\bX^T\bX)^{-1}\bX^T||_{op}^2||A||_{op}^2 \right]\right) \\
    &=  ||(\bX^T\bX)^{-1}\bX^T||_{op}^2 \min\left(\E\left[ ||A||_{F}^2  \right], d\E\left[||A||_{op}^2  \right]\right)
\end{align*}
We have $\E\left[ ||A||_{F}^2  \right] = n$ and $\E\left[||A||_{op}^2  \right] = 1$. Also, we are in the setting where $n > d$ to have a well defined regressor. Therefore, we obtain
\[
     \E\left[ ||(\bX^T\bX)^{-1}\bX^TA||_{F}^2  \right]\le d ||(\bX^T\bX)^{-1}\bX^T||_{op}^2 
\]
\end{proof}

\subsection{LLP, Bag-loss}\label{LLP-bag}

\begin{theorem*}[full version of Theorem \ref{thm:LLP-bag-UB1}]

For $\hth$ as in \eqref{eq:bagloss}, for a given bagging $B$ with bagging matrix $S$, we have
\begin{align*}
  \E\left[ \|\hth-\tth\|_2^2  \right] \leq   \sigma^2 \left(\frac{\lambda_{max}((S\bX)^T S\bX)}{\lambda_{min}((S\bX)^T S\bX)}\right)^2 \left(\sum_{l = 1}^m \frac{1}{|B_l|}\right)
\end{align*}
For equal sized bags of size $k$, this simplifies to
\begin{align*}\label{eq:LLP-bag-UB1}
    \E\left[ \|\hth-\tth\|_2^2  \right] \leq 
    \sigma^2 \frac{m}{k} \left( 
    	\frac
    	{\lambda_{max} ((S\bX)^T S\bX)^{-1}}
    	{\lambda_{min} ((S\bX)^T S\bX)^{-1}} 
    \right)^2. 
\end{align*}
\end{theorem*}

\begin{proof}
We start by proving the following lemma

\begin{lemma}
\begin{align}
    \E\left[ \|\hth-\tth\|_2^2  \right]
    =&
    \sigma^2 \|((S\bX)^T S\bX)^{-1} (S\bX)^T (S S^T)^{1/2}\|_F^2\,.
\end{align}
\end{lemma}

\begin{proof}
    The minimizer of the bag-level loss in matrix form is
\begin{align*}\label{eq:bag_loss_matrix}
    \hth &=\argmin_{\theta} \frac{1}{m}\|S \by - S\bX\theta\|_2^2\\
    &= (\bX^T S^T S\bX)^{-1} \bX^T S^T S y.
\end{align*}

    By rearranging the terms, we have
\begin{align*}
\hth - \tth &= ((S\bX)^T S\bX)^{-1} \bX^T S^T S\by - \tth\\
&= ((S\bX)^T S\bX)^{-1} \bX^T S^T S\bX\tth -\tth\\ &+ ((S\bX)^T S\bX)^{-1} \bX^T S^T S \eps\\
&= ((S\bX)^T S\bX)^{-1} \bX^T S^T S \eps 
\end{align*}

Since $\eps$ is independent of $\bX$l, with
$\E[\eps] = 0$, and $\E[\eps\eps^T]=\sigma^2\Iden$, we have
\begin{align*}
\E\left[\|\hth - \tth\|_2^2  \right] = \sigma^2 tr(((S\bX)^T S\bX)^{-1} (S\bX)^T S S^T  (S\bX) ((S\bX)^T S\bX)^{-1})
\end{align*}

By definition, $S S^T = \text{Diag}(\{\frac{1}{|B_1|}, \frac{1}{|B_2|}, \dots, \frac{1}{|B_m|}\})$ and the expression simplifies to give:
\begin{equation*}
    \E\left[\|\hth - \tth\|_2^2 \right] = \sigma^2 \|((S\bX)^T S\bX)^{-1} (S\bX)^T (S S^T)^{1/2}\|_F^2 
\end{equation*}
\end{proof}

Now we upper bound the RHS.

\begin{align*}
\E\left[\|\hth - \tth\|_2^2 \right] &= \sigma^2 \|((S\bX)^T S\bX)^{-1}  (S\bX)^T (S S^T)^{1/2}\|_F^2 \\
&\leq \sigma^2 \|((S\bX)^T S\bX)^{-1} (S\bX)^T\|_{op}^2 \| (SS^T)^{1/2}\|_F^2 \\
&= \sigma^2 \|((S\bX)^T S\bX)^{-1} (S\bX)^T\|_{op}^2 \left(\sum_{l = 1}^m \frac{1}{|B_l|}\right) \\
&\leq \sigma^2 \|((S\bX)^T S\bX)^{-1}\|_{op}^2 \|(S\bX)^T\|_{op}^2 \left(\sum_{l = 1}^m \frac{1}{|B_l|}\right) \\
&\leq \sigma^2 \left(\frac{\lambda_{max}((S\bX)^T S\bX)}{\lambda_{min}((S\bX)^T S\bX)}\right)^2 \left(\sum_{l = 1}^m \frac{1}{|B_l|}\right)
\end{align*}
\end{proof}

\subsection{MIR, Aggregate-loss}\label{MIR-agg}
\begin{theorem*}[full version of Theorem \ref{thm:MIR-agg-UB1}]
For $\hth$ in \eqref{eq:aggloss}, given a bagging $B$ with bagging matrix $S$,
\begin{align*}
    \E\left[ \|\hth-\tth\|_2^2 \right]
    \leq \|((S\bX)^T S\bX)^{-1} (S\bX)^T\|_{op}^2 \left(\sum_{l=1}^m \left(\frac{\sum_{i\in B_l} \ty^2_i}{|B_l|}\right) - \sum_{l=1}^m \left(\frac{\sum_{i\in B_l} \ty_i}{|B_l|}\right)^2  + \sigma^2 n \right)
\end{align*}
For equal sized bags, this simplifies to
\begin{align}
    &\E\left[ \|\hth-\tth\|_2^2 \right]
    \leq \frac{1}{k} \|((S\bX)^T S\bX)^{-1} (S\bX)^T\|_{op}^2 \nonumber \left( \sum_{l=1}^m \sum_{\ty_i\in B_l} (\ty_i - \mu_l)^2  + \sigma^2nk \right),
\end{align}
\end{theorem*}

\begin{proof}
\begin{align*}\label{eq:bag_loss_matrix}
    \hth &=\argmin_{\theta} \frac{1}{m}\|A \by - S\bX\theta\|_2^2 \\
     &= (\bX^T S^T S\bX)^{-1} \bX^T S^T A \by.
\end{align*}

    By rearranging the terms, we have
\begin{align*}
\hth - \tth &= ((S\bX)^T S\bX)^{-1} \bX^T S^T A\by - \tth\\
&= ((S\bX)^T S\bX)^{-1} \bX^T S^T A\bX\tth -\tth + ((S\bX)^T S\bX)^{-1} \bX^T S^T A \eps
\end{align*}
$\eps$ is independent of $\bX$ with
$\E[\eps] = 0$ and $\E[\eps\eps^T]=\sigma^2\Iden$. Also, $\E[A] = S$, and $\eps,A$ are independent. Hence, 
\begin{align*}
\E\left[\|\hth - \tth\|_2^2 \right] &= \E\left[\|((S\bX)^T S\bX)^{-1} (S\bX)^T AX\tth - ((S\bX)^T S\bX)^{-1} (S\bX)^T S\bX\tth + ((S\bX)^T S\bX)^{-1} \bX^T S^T A \eps \|_2^2\right] \\
&\leq \|((S\bX)^T S\bX)^{-1} (S\bX)^T\|_{op}^2 \E[\|(A\bX\tth - S\bX\tth) + A\eps\|_2^2] \\
&\leq \|((S\bX)^T S\bX)^{-1} (S\bX)^T\|_{op}^2 \left( \E[\|A\bX\tth - S\bX\tth\|_2^2] + \E [\|A\eps\|_2^2] \right) \\
&\leq \|((S\bX)^T S\bX)^{-1} (S\bX)^T\|_{op}^2 \left( \E[\|A\Tilde{y} - S\Tilde{y}\|_2^2] + \E [\|A\eps\|_2^2] \right)
\end{align*}
We now analyse $ \E[\|A\Tilde{y} - S\Tilde{y}\|_2^2]$ in the lemma below.
\begin{lemma}
\begin{align*}
    \E[\|A\Tilde{y} - S\Tilde{y}\|_2^2] = \sum_{l=1}^m \left(\frac{\sum_{i\in B_l} \ty^2_i}{|B_l|}\right) - \sum_{l=1}^m \left(\frac{\sum_{i\in B_l} \ty_i}{|B_l|}\right)^2
\end{align*}
\end{lemma}
\begin{proof}

\begin{align*}
    \E[\|A\Tilde{y} - S\Tilde{y}\|_2^2] &= \E [ (A\Tilde{y} - S\Tilde{y})^T(A\Tilde{y} - S\Tilde{y})] \\
    &= \E [ ||A\Tilde{y}||^2 + ||S\Tilde{y}||^2 - 2\Tilde{y}^TS^TA\Tilde{y} ] \\
    &= \E [ ||A\Tilde{y}||^2 ] + \E [ ||S\Tilde{y}||^2] - 2 \E [\Tilde{y}^TS^TA\Tilde{y} ] \\
    &= \E [ ||A\Tilde{y}||^2 ] + \E [ ||S\Tilde{y}||^2] - 2 \E [\Tilde{y}^TS^TS{y} ] \\
    &= \E [ ||A\Tilde{y}||^2 ] + \E [ ||S\Tilde{y}||^2] - 2 \E [||S\Tilde{y}||^2 ] \\
    &= \E [ ||A\Tilde{y}||^2 ] - \E [ ||S\Tilde{y}||^2]  \\
    &= \E [ ||A\Tilde{y}||^2 ] -  ||S\Tilde{y}||^2 
\end{align*}
We now analyse $\E [ ||A\Tilde{y}||^2 ]$
\begin{align*}
    A\ty &= \left[\ty_{\Gamma(B_1)}, \ldots, \ty_{\Gamma(B_m)}\right]^T\\
    \implies \ty^TA^T  A \ty &= \sum_{l=1}^{l=m} \ty^2_{\Gamma(B_l)}
\end{align*}
Then we have 
\begin{align*}
    \E \left[ \ty^TA^T  A \ty \right] &= \E \left[\sum_{l=1}^{l=m} \ty^2_{\Gamma(B_l)}\right]\\
&=
\sum_{l=1}^m \left(\frac{\sum_{i\in B_l} \ty^2_i}{|B_l|}\right)
\end{align*}
For equal size bags it simplifies to $\frac{||\ty||^2}{k}$. We now analyse Term 2 $||S\Tilde{y}||^2$
\begin{align*}
    S\ty &= \left[ \frac{\sum_{i\in B_1} \ty_i}{|B_1|}, \ldots, \frac{\sum_{i\in B_m} \ty_i}{|B_m|}\right]^T\\
    \implies \ty^TS^T  S \ty &= \sum_{l=1}^m \left(\frac{\sum_{i\in B_l} \ty_i}{|B_l|}\right)^2
\end{align*}
For equal size bags this simplifies to $\sum_{l=1}^m \left(\frac{\sum_{i\in B_l} \ty_i}{k}\right)^2$.
\end{proof}

It is easy to see that $\E [\|A\eps\|_2^2] = n\sigma^2$. Combining this with the above lemma, we are done.

\end{proof}

\subsection{Privacy}\label{appendix:privacy}

In this section, we quantify the additional loss in utility incurred due to label-DP guarantees, for each setting we consider (instance-MIR, bag-LLP, and aggregate MIR). We give full versions of the theorems stated in Section \ref{privacy}, along with the proofs.

\subsubsection{MIR, instance-level}

    \begin{theorem*}[full version of Theorem \ref{thm:instance-mir-loss-priv}]
There exists a bagging $B~\text{with }  |B_l| = k, \forall l \in [m]$, satisfying $(\epsilon,\delta)$ label-DP, such that for $\hth$ in \eqref{eq:instanceloss}, we have
\begin{align*}
    &\E\left[ ||\hth-\tth||_2^2 \right]
    \leq ||(\bX^T\bX)^{-1}\bX^T||_{op}^2\left(2\left( \kmeansopt + n\left(1-\frac{1}{k}\right)\alpha^2 \right)  +  d \left( \sigma^2 + \frac{\alpha^2}{k^2} \right) \right), \nonumber
\end{align*}
where $\alpha^2 = \frac{16R^2\log\left(\frac{1.25}{\delta/2}\right) }{\epsilon^2}$, and $OPT$ is the objective value of the optimal $k$-means clustering over $\ty$.
\end{theorem*}
\begin{proof}
The error due to privacy can be decomposed into two parts. 

We need to add noise to the bag-labels before releasing them. MIR outputs one label at random, hence the sensitivity of the output is $2R$. Due to privacy amplification via subsampling \citep{balle2018privacyamplificationsubsamplingtight, steinke2022compositiondifferentialprivacy}, and the fact that $\epsilon << n$ in our setting, we add $\normal\left(0,\frac{\alpha^2}{k^2} \right)$ noise to the bag-label value to ensure $\left(\frac{\epsilon}{2}, \frac{\delta}{2}\right)$ label-DP, where $\alpha^2 = \frac{16R^2\log\left(\frac{1.25}{\delta/2}\right) }{\epsilon^2}$. Note that we assume addition of $\normal\left(0,\sigma^2 \right)$ noise to each $\ty_i$. Adding $\normal\left(0,\frac{\alpha^2}{k^2} \right)$ to each bag-label is equivalent to adding $\normal\left(0,\frac{\alpha^2}{k^2} \right)$ to each label $y_i$, hence leading to a total noise of $\normal\left(0,\sigma^2 + \frac{\alpha^2}{k^2} \right)$ to each $\ty_i$, leading to an additional error of $d\frac{\alpha^2}{k^2}$ over the intital $d\sigma^2$.  

In addition, since the objective here is a label-dependent clustering, we must use a differentially private $k$-means algorithm, leading to additional loss in utility. Adding $\normal\left(0,\alpha^2 \right)$ noise to each label, and then find an optimal clustering over the noise labels, satisfies $\left(\frac{\epsilon}{2}, \frac{\delta}{2}\right)$ label-DP by postprocessing. If $OPT$ is the objective value of the optimal $k$-means clustering over $\ty$, this private clustering method will lead to an additional error of $\left(1-\frac{1}{k}\right) \alpha^2$, due to Lemma \ref{noisyclustering}. 

Now, we have two queries, each of which are $\left(\frac{\epsilon}{2}, \frac{\delta}{2}\right)$ label-DP, ensuring $(\epsilon, \delta)$ label-DP in total due to composition.
\end{proof}

\paragraph{Private clustering} Note that it is possible to further reduce the error $n\left(1-\frac{1}{k}\right) \alpha^2$ due to private clustering. Note that the above method for private clustering satisfies the more stringent notion of local-DP \citep{bebensee2019localdifferentialprivacytutorial}, while we only need to satisfy the standard notion of central-DP. Hence, while it is easy to analyse, we can potentially find a much more accurate private clustering mechanism, suitably modifying existing algorithms in the rich literature on differentially-private $k$-means clustering \citep{su2015differentiallyprivatekmeansclustering, Lu_2020}, for the special case of a single dimension.

\subsubsection{LLP, bag-level}

\begin{theorem*}[full version of Theorem \ref{thm:bag-llp-loss-priv}]
There exists a bagging $B~\text{with }  |B_l| = k, \forall l \in [m]$, satisfying $(\epsilon,\delta)$ label-DP, such that for $\hth$ in \eqref{eq:bagloss}, we have
\begin{align}
    &\E\left[ \|\hth-\tth\|_2^2 \right]
    = OPT \left(  \sigma^2 + \frac{\alpha^2}{k} \right) \frac{m}{k},\nonumber
\end{align}
where $\alpha^2 = \frac{4R^2\log\left(\frac{1.25}{\delta}\right) }{\epsilon^2}$, and $OPT$ is the optimal value of $\left(\frac{\lambda_{max}(f(X))}{\lambda_{min}(f(X))}\right)^2$.
\end{theorem*}
\begin{proof}
     In this case, the optimal bagging strategy in independent of the labels. Hence, we just need to add noise to the bag-labels before releasing them, and not add noise for a private clustering of the labels. Each bag label here is the mean of $k$ labels, hence the sensitivity of the output is $\frac{2R}{k}$. We add $\normal\left(0,\frac{\alpha^2}{k^2} \right)$ noise to the label value to ensure $(\epsilon,\delta)$ label-DP, where $\alpha^2 = \frac{4R^2\log\left(\frac{1.25}{\delta}\right) }{\epsilon^2}$. This is equivalent to adding $\normal\left(0,\frac{\alpha^2}{k} \right)$ noise to each of the $k$ labels, and then averaging them. Note that we assume addition of $\normal\left(0,\sigma^2 \right)$ noise to each $\ty_i$. Adding $\normal\left(0,\frac{\alpha^2}{k} \right)$ to each label $y_i$, leads to a total noise of $\normal\left(0,\sigma^2 + \frac{\alpha^2}{k} \right)$ to each $\ty_i$, leading to an additional error of $\frac{\alpha^2}{k} \frac{m}{k}$ over the intital $\sigma^2 \frac{m}{k}$.

\end{proof}

\subsubsection{MIR, aggregate-level}

Theorem \ref{thm:MIR-agg-UB1} shows that, for $\hth$ in \eqref{eq:aggloss}, given a bagging $B$, with equal sized bags, we have
\begin{align}
    &\E\left[ \|\hth-\tth\|_2^2 \right]
    \leq \frac{1}{k} \|((S\bX)^T S\bX)^{-1} (S\bX)^T\|_{op}^2 \nonumber \left( \sum_{l=1}^m \sum_{\ty_i\in B_l} (\ty_i - \mu_l)^2  + \sigma^2nk \right),
\end{align}

If we want a private bagging $B$, the error due to privacy can be decomposed into two parts. We need to add noise to the bag-labels before releasing them. As in the case of instance-MIR, we add $\normal\left(0,\frac{\alpha^2}{k^2} \right)$ noise to the bag-labels value to ensure $(\epsilon,\delta)$ label-DP, where $\alpha^2 = \frac{4R^2\log\left(\frac{1.25}{\delta}\right) }{\epsilon^2}$, leading to an additional error of $nk\frac{\alpha^2}{k^2}$ over the intital $nk\sigma^2$.

Now, there are two terms that contribute to the clustering error, term 1 $\left( \|((S\bX)^T S\bX)^{-1} (S\bX)^T\|_{op}^2\right)$, and term 2 $\left(\sum_{l=1}^m \sum_{\ty_i\in B_l} (\ty_i - \mu_l)^2\right)$. Term 1 is involved in bag-LLP, and minimizes the condition number of the bag-centroids. Term 2 is also involved in instance-MIR, and minimizes a label-dependent $k$-means clustering objective. If we minimize Term 1, the optimal bagging strategy in independent of the labels. Hence, we just need to add noise to the bag-labels before releasing them, and not add noise for a private clustering of the labels. However, in this case, the value of Term 2 could be suboptimal. 

If we minimize Term 2, we must use a differentially private $k$-means algorithm, leading to additional loss in utility. Adding $\normal\left(0,\alpha^2 \right)$ noise to each label, and then find an optimal clustering over the noise labels, satisfies $(\epsilon,\delta)$ label-DP. As in the case of instance MIR, this private clustering method will lead to an additional error of $n\left(1-\frac{1}{k}\right) \alpha^2$. Note that since we now have two private queries, we would have to split the privacy budget amongst them. However, minimizing term 2 might lead to a suboptimal value of Term 1.

\section{ADDITIONAL EXPERIMENTS}\label{appendix:experiments}
\begin{table}[]
\begin{tiny}
    \centering
    \begin{tabular}{p{3cm}p{2cm}p{2cm}p{5cm}p{4cm}}%
\toprule
 &  &  &  &  \\
Data & k & $\sigma$ & Bagging Method & $\|\hth-\tth\|_2^2$ \\
\midrule
\multirow[c]{16}{*}{Isotropic} & \multirow[c]{8}{*}{10} & \multirow[c]{4}{*}{0.5} & Instance $k$-means & $0.010693 \pm 0.00167$ \\
 &  &  & Label $k$-means& $0.044320 \pm 0.00720$ \\
 &  &  & Label $k$-means super-bags & $0.040845 \pm 0.01104$ \\
 &  &  & Random& $0.022352 \pm 0.00447$ \\
\cline{3-5}
 &  & \multirow[c]{4}{*}{2} & Instance $k$-means & $0.037875 \pm 0.00494$ \\
 &  &  & Label $k$-means& $0.056199 \pm 0.01042$ \\
 &  &  & Label $k$-means super-bags & $0.059399 \pm 0.01304$ \\
 &  &  & Random& $0.053995 \pm 0.01119$ \\
\cline{2-5} \cline{3-5}
 & \multirow[c]{8}{*}{50} & \multirow[c]{4}{*}{0.5} & Instance $k$-means & $0.046242 \pm 0.00773$ \\
 &  &  & Label $k$-means& $0.064936 \pm 0.01016$ \\
 &  &  & Label $k$-means super-bags & $0.058051 \pm 0.00631$ \\
 &  &  & Random& $0.057210 \pm 0.00981$ \\
\cline{3-5}
 &  & \multirow[c]{5}{*}{2} & Instance $k$-means & $0.056337 \pm 0.01002$ \\
 &  &  & Label $k$-means& $0.065491 \pm 0.00853$ \\
 &  &  & Label $k$-means super-bags & $0.061981 \pm 0.00991$ \\
 &  &  & Random& $0.065836 \pm 0.01079$ \\
\cline{1-5} \cline{2-5} \cline{3-5}
\multirow[c]{20}{*}{\shortstack[l]{Non-isotropic \\ (Independent)}} & \multirow[c]{10}{*}{10} & \multirow[c]{5}{*}{0.5} & Instance $k$-means & $0.014946 \pm 0.00421$ \\
 &  &  & Label $k$-means& $0.040369 \pm 0.00990$ \\
 &  &  & Label $k$-means super-bags & $0.042778 \pm 0.00804$ \\
 &  &  & Random& $0.020230 \pm 0.00506$ \\
 &  &  & Scaled Instance $k$-means & $0.012608 \pm 0.00354$ \\
\cline{3-5}
 &  & \multirow[c]{5}{*}{2} & Instance $k$-means & $0.039141 \pm 0.00884$ \\
 &  &  & Label $k$-means& $0.048532 \pm 0.01083$ \\
 &  &  & Label $k$-means super-bags & $0.052560 \pm 0.01105$ \\
 &  &  & Random& $0.058208 \pm 0.00860$ \\
 &  &  & Scaled Instance $k$-means & $0.042403 \pm 0.00573$ \\
\cline{2-5} \cline{3-5}
 & \multirow[c]{10}{*}{50} & \multirow[c]{5}{*}{0.5} & Instance $k$-means & $0.041916 \pm 0.00736$ \\
 &  &  & Label $k$-means& $0.062490 \pm 0.00929$ \\
 &  &  & Label $k$-means super-bags & $0.060436 \pm 0.01054$ \\
 &  &  & Random& $0.055356 \pm 0.01085$ \\
 &  &  & Scaled Instance $k$-means & $0.047906 \pm 0.00964$ \\
\cline{3-5}
 &  & \multirow[c]{5}{*}{2} & Instance $k$-means & $0.059583 \pm 0.00788$ \\
 &  &  & Label $k$-means& $0.062350 \pm 0.01028$ \\
 &  &  & Label $k$-means super-bags & $0.062662 \pm 0.01306$ \\
 &  &  & Random& $0.065602 \pm 0.00934$ \\
 &  &  & Scaled Instance $k$-means & $0.059133 \pm 0.01235$ \\
\cline{1-5} \cline{2-5} \cline{3-5}
\multirow[c]{20}{*}{\shortstack[l]{Non-isotropic \\ (Non-independent)}} & \multirow[c]{10}{*}{10} & \multirow[c]{5}{*}{0.5} & Instance $k$-means & $0.031268 \pm 0.00649$ \\
 &  &  & Label $k$-means& $0.052303 \pm 0.01065$ \\
 &  &  & Label $k$-means super-bags & $0.049302 \pm 0.00531$ \\
 &  &  & Random& $0.034642 \pm 0.01052$ \\
 &  &  & Scaled Instance $k$-means & $0.022451 \pm 0.00636$ \\
\cline{3-5}
 &  & \multirow[c]{5}{*}{2} & Instance $k$-means & $0.043493 \pm 0.00732$ \\
 &  &  & Label $k$-means& $0.054761 \pm 0.01151$ \\
 &  &  & Label $k$-means super-bags & $0.056316 \pm 0.01127$ \\
 &  &  & Random& $0.055723 \pm 0.01053$ \\
 &  &  & Scaled Instance $k$-means & $0.039650 \pm 0.00781$ \\
\cline{2-5} \cline{3-5}
 & \multirow[c]{10}{*}{50} & \multirow[c]{5}{*}{0.5} & Instance $k$-means & $0.052643 \pm 0.01071$ \\
 &  &  & Label $k$-means& $0.060606 \pm 0.00677$ \\
 &  &  & Label $k$-means super-bags & $0.059758 \pm 0.00977$ \\
 &  &  & Random& $0.057136 \pm 0.00876$ \\
 &  &  & Scaled Instance $k$-means & $0.046376 \pm 0.00642$ \\
\cline{3-5}
 &  & \multirow[c]{5}{*}{2} & Instance $k$-means & $0.058460 \pm 0.01074$ \\
 &  &  & Label $k$-means& $0.060828 \pm 0.00811$ \\
 &  &  & Label $k$-means super-bags & $0.065220 \pm 0.00745$ \\
 &  &  & Random& $0.067064 \pm 0.01064$ \\
 &  &  & Scaled Instance $k$-means & $0.059597 \pm 0.00908$ \\
\cline{1-5} \cline{2-5} \cline{3-5}
\bottomrule
\end{tabular}
\end{tiny}
    \caption{Aggregate-MIR}
    \label{tab:app_agg_mir}
    
\end{table}

\begin{table}[]
\begin{tiny}
    \centering
    \begin{tabular}{p{3cm} p{2cm} p{2cm}p{5cm}p{4cm}}
\toprule
 &  &  &  &  \\
Data & k & $\sigma$ & Bagging Method & $\|\hth-\tth\|_2^2$ \\
\midrule
\multirow[c]{16}{*}{Isotropic} & \multirow[c]{8}{*}{10} & \multirow[c]{4}{*}{0.5} & Instance $k$-means & $0.007562 \pm 0.00137$ \\
 &  &  & Label $k$-means& $0.043625 \pm 0.00722$ \\
 &  &  & Label $k$-means super-bags & $0.044586 \pm 0.00906$ \\
 &  &  & Random& $0.009745 \pm 0.00206$ \\
\cline{3-5}
 &  & \multirow[c]{4}{*}{2} & Instance $k$-means & $0.014722 \pm 0.00329$ \\
 &  &  & Label $k$-means& $0.056195 \pm 0.01101$ \\
 &  &  & Label $k$-means super-bags & $0.056651 \pm 0.01085$ \\
 &  &  & Random& $0.026405 \pm 0.00502$ \\
\cline{2-5} \cline{3-5}
 & \multirow[c]{8}{*}{50} & \multirow[c]{4}{*}{0.5} & Instance $k$-means & $0.037432 \pm 0.00721$ \\
 &  &  & Label $k$-means& $0.063826 \pm 0.00800$ \\
 &  &  & Label $k$-means super-bags & $0.058686 \pm 0.01111$ \\
 &  &  & Random& $0.046269 \pm 0.00830$ \\
\cline{3-5}
 &  & \multirow[c]{4}{*}{2} & Instance $k$-means & $0.040709 \pm 0.00964$ \\
 &  &  & Label $k$-means& $0.063859 \pm 0.00486$ \\
 &  &  & Label $k$-means super-bags & $0.058983 \pm 0.00880$ \\
 &  &  & Random& $0.049042 \pm 0.00872$ \\
\cline{1-5} \cline{2-5} \cline{3-5}
\multirow[c]{20}{*}{\shortstack[l]{Non-isotropic \\ (Independent)}} & \multirow[c]{10}{*}{10} & \multirow[c]{5}{*}{0.5} & Instance $k$-means & $0.009739 \pm 0.00201$ \\
 &  &  & Label $k$-means& $0.042496 \pm 0.00626$ \\
 &  &  & Label $k$-means super-bags & $0.044571 \pm 0.00929$ \\
 &  &  & Random& $0.010518 \pm 0.00339$ \\
 &  &  & Scaled Instance $k$-means & $0.008552 \pm 0.00191$ \\
\cline{3-5}
 &  & \multirow[c]{5}{*}{2} & Instance $k$-means & $0.018930 \pm 0.00425$ \\
 &  &  & Label $k$-means& $0.049482 \pm 0.01074$ \\
 &  &  & Label $k$-means super-bags & $0.055759 \pm 0.01066$ \\
 &  &  & Random& $0.030314 \pm 0.00652$ \\
 &  &  & Scaled Instance $k$-means & $0.014849 \pm 0.00286$ \\
\cline{2-5} \cline{3-5}
 & \multirow[c]{10}{*}{50} & \multirow[c]{5}{*}{0.5} & Instance $k$-means & $0.036923 \pm 0.00536$ \\
 &  &  & Label $k$-means& $0.059834 \pm 0.00598$ \\
 &  &  & Label $k$-means super-bags & $0.062452 \pm 0.01025$ \\
 &  &  & Random& $0.039461 \pm 0.00760$ \\
 &  &  & Scaled Instance $k$-means & $0.038586 \pm 0.00784$ \\
\cline{3-5}
 &  & \multirow[c]{5}{*}{2} & Instance $k$-means & $0.043048 \pm 0.01045$ \\
 &  &  & Label $k$-means& $0.058143 \pm 0.01113$ \\
 &  &  & Label $k$-means super-bags & $0.059907 \pm 0.00812$ \\
 &  &  & Random& $0.054860 \pm 0.00659$ \\
 &  &  & Scaled Instance $k$-means & $0.045390 \pm 0.00617$ \\
\cline{1-5} \cline{2-5} \cline{3-5}
\multirow[c]{20}{*}{\shortstack[l]{Non-isotropic \\ (Non-independent)}} & \multirow[c]{10}{*}{10} & \multirow[c]{5}{*}{0.5} & Instance $k$-means & $0.032367 \pm 0.00835$ \\
 &  &  & Label $k$-means& $0.052438 \pm 0.00936$ \\
 &  &  & Label $k$-means super-bags & $0.050445 \pm 0.01255$ \\
 &  &  & Random& $0.024585 \pm 0.00755$ \\
 &  &  & Scaled Instance $k$-means & $0.024811 \pm 0.00498$ \\
\cline{3-5}
 &  & \multirow[c]{5}{*}{2} & Instance $k$-means & $0.033099 \pm 0.01050$ \\
 &  &  & Label $k$-means& $0.057081 \pm 0.00955$ \\
 &  &  & Label $k$-means super-bags & $0.057327 \pm 0.01297$ \\
 &  &  & Random& $0.032676 \pm 0.00675$ \\
 &  &  & Scaled Instance $k$-means & $0.029420 \pm 0.00755$ \\
\cline{2-5} \cline{3-5}
 & \multirow[c]{10}{*}{50} & \multirow[c]{5}{*}{0.5} & Instance $k$-means & $0.051425 \pm 0.00895$ \\
 &  &  & Label $k$-means& $0.061918 \pm 0.00820$ \\
 &  &  & Label $k$-means super-bags & $0.058320 \pm 0.01040$ \\
 &  &  & Random& $0.048222 \pm 0.01074$ \\
 &  &  & Scaled Instance $k$-means & $0.049910 \pm 0.00773$ \\
\cline{3-5}
 &  & \multirow[c]{5}{*}{2} & Instance $k$-means & $0.051430 \pm 0.00661$ \\
 &  &  & Label $k$-means& $0.065289 \pm 0.01090$ \\
 &  &  & Label $k$-means super-bags & $0.069147 \pm 0.01071$ \\
 &  &  & Random& $0.059075 \pm 0.00885$ \\
 &  &  & Scaled Instance $k$-means & $0.047859 \pm 0.00678$ \\
\cline{1-5} \cline{2-5} \cline{3-5}
\bottomrule
\end{tabular}
\end{tiny}
    \caption{Bag-LLP}
    \label{tab:app_bag_llp}
\end{table}

\begin{table}[]
\begin{tiny}
    \centering
    \begin{tabular}{p{3cm}p{2cm}p{2cm}p{5cm}p{4cm}}
\toprule
 &  &  &  & $\|\hth-\tth\|_2^2$ \\
Data & k & $\sigma$ & Bagging Method &  \\
\midrule
\multirow[c]{12}{*}{Isotropic} & \multirow[c]{6}{*}{10} & \multirow[c]{3}{*}{0.5} & Instance $k$-means & $0.008894 \pm 0.00168$ \\
 &  &  & Label $k$-means& $0.007597 \pm 0.00197$ \\
 &  &  & Random& $0.007997 \pm 0.00174$ \\
\cline{3-5}
 &  & \multirow[c]{3}{*}{2} & Instance $k$-means & $0.019629 \pm 0.00410$ \\
 &  &  & Label $k$-means& $0.010983 \pm 0.00239$ \\
 &  &  & Random& $0.010078 \pm 0.00190$ \\
\cline{2-5} \cline{3-5}
 & \multirow[c]{6}{*}{50} & \multirow[c]{3}{*}{0.5} & Instance $k$-means & $0.039916 \pm 0.00828$ \\
 &  &  & Label $k$-means& $0.040155 \pm 0.00986$ \\
 &  &  & Random& $0.044420 \pm 0.00472$ \\
\cline{3-5}
 &  & \multirow[c]{3}{*}{2} & Instance $k$-means & $0.049003 \pm 0.01167$ \\
 &  &  & Label $k$-means& $0.040044 \pm 0.00608$ \\
 &  &  & Random& $0.040281 \pm 0.00600$ \\
\cline{1-5} \cline{2-5} \cline{3-5}
\multirow[c]{16}{*}{\shortstack[l]{Non-isotropic \\ (Independent)}} & \multirow[c]{8}{*}{10} & \multirow[c]{4}{*}{0.5} & Instance $k$-means & $0.008672 \pm 0.00215$ \\
 &  &  & Label $k$-means& $0.007790 \pm 0.00158$ \\
 &  &  & Random& $0.008808 \pm 0.00174$ \\
 &  &  & Scaled Instance $k$-means & $0.009683 \pm 0.00102$ \\
\cline{3-5}
 &  & \multirow[c]{4}{*}{2} & Instance $k$-means & $0.018395 \pm 0.00421$ \\
 &  &  & Label $k$-means& $0.012217 \pm 0.00205$ \\
 &  &  & Random& $0.011335 \pm 0.00198$ \\
 &  &  & Scaled Instance $k$-means & $0.022363 \pm 0.00499$ \\
\cline{2-5} \cline{3-5}
 & \multirow[c]{8}{*}{50} & \multirow[c]{4}{*}{0.5} & Instance $k$-means & $0.042065 \pm 0.00686$ \\
 &  &  & Label $k$-means& $0.041108 \pm 0.00867$ \\
 &  &  & Random& $0.038124 \pm 0.00552$ \\
 &  &  & Scaled Instance $k$-means & $0.037391 \pm 0.00674$ \\
\cline{3-5}
 &  & \multirow[c]{4}{*}{2} & Instance $k$-means & $0.043934 \pm 0.00901$ \\
 &  &  & Label $k$-means& $0.041059 \pm 0.00527$ \\
 &  &  & Random& $0.044340 \pm 0.00826$ \\
 &  &  & Scaled Instance $k$-means & $0.047298 \pm 0.00768$ \\
\cline{1-5} \cline{2-5} \cline{3-5}
\multirow[c]{16}{*}{\shortstack[l]{Non-isotropic \\ (Non-independent)}} & \multirow[c]{8}{*}{10} & \multirow[c]{4}{*}{0.5} & Instance $k$-means & $0.023122 \pm 0.00747$ \\
 &  &  & Label $k$-means& $0.023248 \pm 0.00916$ \\
 &  &  & Random& $0.022115 \pm 0.00565$ \\
 &  &  & Scaled Instance $k$-means & $0.019744 \pm 0.00628$ \\
\cline{3-5}
 &  & \multirow[c]{4}{*}{2} & Instance $k$-means & $0.035530 \pm 0.01027$ \\
 &  &  & Label $k$-means& $0.027272 \pm 0.00708$ \\
 &  &  & Random& $0.026394 \pm 0.00626$ \\
 &  &  & Scaled Instance $k$-means & $0.034814 \pm 0.00768$ \\
\cline{2-5} \cline{3-5}
 & \multirow[c]{8}{*}{50} & \multirow[c]{4}{*}{0.5} & Instance $k$-means & $0.049454 \pm 0.00978$ \\
 &  &  & Label $k$-means& $0.048404 \pm 0.00920$ \\
 &  &  & Random& $0.048654 \pm 0.01101$ \\
 &  &  & Scaled Instance $k$-means & $0.051057 \pm 0.00644$ \\
\cline{3-5}
 &  & \multirow[c]{4}{*}{2} & Instance $k$-means & $0.049799 \pm 0.00843$ \\
 &  &  & Label $k$-means& $0.045538 \pm 0.00981$ \\
 &  &  & Random& $0.047661 \pm 0.00710$ \\
 &  &  & Scaled Instance $k$-means & $0.048617 \pm 0.00801$ \\
\cline{1-5} \cline{2-5} \cline{3-5}
\bottomrule
\end{tabular}
\end{tiny}
    \caption{Instance-MIR}
    \label{tab:app_instance_mir}
\end{table}

In the main paper, we restricted experiments to a dataset drawn from an isotropic distribution. Here, we also consider non-isotropic distributions. We generate datasets in the following way:\sushant{todo, explain how we get $\Sigma$ in both independent and non-independent cases, format tables}
\begin{itemize}
\item \emph{Isotropic}: We independently sample a set $\mathcal{X}$ containing $n$ $d$-dimensional points from $\mathcal{N}(0, I)$.
\item \emph{Non-isotropic (Independent)}: We sample $d$ independent values $\{\lambda_1, \cdots, \lambda_d\}$ from a uniform distribution $U(0.1, 10)$ to be the eigenvalues of the $\Sigma$, which is diagonal matrix. 
\item \emph{Non-isotropic (Non-independent)}: We sample each entry of a Cholesky matrix $M$ of size $d\times d$ from $\mathcal{N}(0,1)$. We then compute the covariance matrix $M^TM$ and apply a linear transformation to feature vectors $x$ sampled from $\mathcal{N}(0, I)$ using $M$. The resulting set of vectors is non-isotropic with correlated features.
\end{itemize}

Once we have sampled feature vectors of the form $X \in \mathbb{R}^{n\times d}$, we sample a random groud truth model $\theta^*$ from the standard $d$-dimensional Gaussian distribution. This true model is then used to generate the true labels $\Tilde{y}$. We add noise to $\Tilde{y}$ to generate $y$. We set $y = X\theta^* + \eps$ where each coordinate of $\eps$ is iid drawn from $N(0, \sigma^2)$ where $\sigma$ is 0.5. We set $n$ to be $50,000$ and $d$ as $32$. We also vary $k$, and use $k = 10, 50$. The result dataset is of the form $(X \in \mathbb{R}^{n\times d}, y \in \mathbb{R}^n)$

We implement 4 bagging mechanisms on each of instance-MIR, aggregate-MIR, and bag-LLP, namely (1) Instance $k$-means, (2) Label $k$-means, (3) Random bagging, and (4) Scaled Instance $k$-means, that scales the dataset $X$ as $\Sigma^{-\frac{1}{2}}X$ to be isotropic, and then finds an optimal $k$-means clustering on the scaled dataset. 
In the tables, we present the mean and standard deviation of the error, calculated over $15$ runs for each experiment. As expected, in most cases for bag-LLP (Table \ref{tab:app_bag_llp}) and aggregate-MIR (Table \ref{tab:app_agg_mir}), scaled instance $k$-means performs better than instance $k$-means, which in turn performs better than random bagging, which in turn performs better than label $k$-means. However, for instance-MIR (Table \ref{tab:app_instance_mir}), all the mechanisms show similar performance, with label $k$-means showing better performance in many cases.

\section{Instance $k$-means}\label{appendix:k_means}

We justify that $k$-means of the instances $X$ is an effective label-agnostic bagging heuristic for each setting we consider (instance-MIR, bag-LLP, and aggregate MIR).

\subsection{MIR, Instance-level}

Note that in our setting of linear regression, $\ty =X\tth$. In other words, $\ty$ is just the projection of $X$ along the axis normal to the hyperplane determined by $\tth$. Hence, finding an optimal $k$-means clustering of $\ty$ is equivalent to minimizing the $k$-means objective of the projection of $X$ along this axis. However, if the labels are not given, this axis is unknown, since $\tth$ is unknown. Hence, in order to do a label-agnostic bagging, one must minimize some objective that simultaneously reduces the $k$-means objective along every direction. We now justify that $k$-means of the instances X is a good heuristic for the same. First, we show that for a given clustering, the $k$-means objective of a dataset is the sum of $k$-means objective of the dataset projected along each coordinate. 

\begin{lemma}
 Consider an orthogonal basis $z_1, \ldots z_d$. Fix a clustering $S$. We can show the following
\begin{equation*}
    \text{k-means}(S(X)) = \sum_{j=1}^d \text{k-means}(S(X_{z_j})),
\end{equation*}
where $\text{k-means}(S(X))$ is the $k$-means clustering objective of $S$ on $X$, and $X_z$ is the projection of $X$ along $z$.
\end{lemma}
\begin{proof} Let $X = \{X_1, \ldots, X_n\}$.
  \begin{align*}
  \text{k-means}(S(X)) &=
      \sum_{l=1}^m \sum_{X_i\in S_l} ||X_i - \mu_l||_2^2 \\ &= \sum_{l=1}^m \sum_{X_i\in S_l} ||X_i||_2^2 + ||\mu_l||_2^2 - 2X_i^T \mu_l\\
      &= \sum_{l=1}^m \sum_{X_i\in S_l} \sum_{j=1}^d \left( {X_{z_j}}_i^2 + \mu_{l_{z_j}}^2 - 2{X_{z_j}}^T \mu_{l_{z_j}} \right) \\
      &= \sum_{j=1}^d \sum_{l=1}^m \sum_{X_i\in S_l}  \left({X_{z_j}}^T - \mu_{l_{z_j}} \right)^2 \\
      &= \sum_{j=1}^d \text{k-means}(S(X_{z_j}))
  \end{align*}
\end{proof}

Given an arbitrary clustering $C$ over $X$ drawn from an isotropic distribution $D$, in expectation the $k$-means clustering objective over $X$ will split equally into $d$ components along each axis (due to symmetry), i.e.,
\begin{equation*}
\expectation[ \text{k-means}(C(X_{z_i}))] = \frac{1}{d}\expectation \left[ \text{k-means}(C(X))\right], \forall i,
\end{equation*}
where the expectation is over $X$ drawn from $D$. Hence, for isotropic distribution $D$, we would expect that the $k$-means clustering objective along each direction to be roughly equal. Hence, we would also expect that setting $S$ to be the optimal $k$-means clustering over $X$ would simultaneously keep the $k$-means clustering objective low along each direction.

However, the above reasoning holds only for an isotropic distribution. For a non-isotropic distribution, directions with large variance will dominate the $k$-means objective, and therefore directions with small variance might then have a relatively large k-means objective. For an isotropic distribution, we avoid the above problem of directions with large variance dominating. However, note that even for a non-isotropic distribution, $\Sigma^{-\frac{1}{2}}X$ is isotropic, where $\Sigma$ is the covariance matrix of the distribution. Essentially, we stretch each direction so that each direction has the same variance. We can now find an optimal $k$-means clustering over $\Sigma^{-\frac{1}{2}}X$. We will then avoid the problem of directions in $X$ with large variance dominating, while also keeping the $k$-means objective along each direction low. A random bagging approach would also avoid the problem of directions with large variance dominating for a non-isotropic distribution. However, the $k$-means objective in every direction will be that of a random clustering, which is sub-optimal.

\subsection{LLP, Bag-level}

Given a bagging with bagging matrix $S$, $SX$ is the matrix representing the dataset consisting of the centroids of each bag. We want to maximize $\lambda_{\min}((SX)^T SX),$ where $(SX)^T SX$ is the sample covariance matrix of $SX$. $\lambda_{\min}$ is the variance along the direction of the corresponding eigen vector (which is also the direction of least variance of dataset $SX$). We now show that finding a bagging $S$ maximizing the variance of $SX$ along a direction is equivalent to finding an optimal $k$-means clustering of $X$ projected on that direction.

\begin{lemma}
Consider a direction $z$, and a centred dataset $X$. Given a bagging $S$ over $X$ with $m$ bags of equal size $k$,
\begin{align*}
    \text{Var}_z(SX) = \frac{1}{k^2} \left(\text{Var}(X_z) - \text{k-means}(S(X_z)) \right),
\end{align*}
\end{lemma}
\begin{proof}
    Say the points are $X_1, \ldots, X_n$, and the projections along $z$ are $x_1, \ldots, x_n$. Let $\mu = 0$ be the mean of $X$, and $\mu_l$ be the mean of $B_l$. The variance of the $SX$ along $z$ is 

\begin{align*}
    \text{Var}(SX_z) &= \sum_{l=1}^m ({\mu_l}_z - \mu_z)^2 \\
    &= \sum_{\ell=1}^m \left(\frac{\sum_{i\in B_\ell} x_i}{k}\right)^2\\
    &= \frac{1}{k^2} \left( \sum_{i=1}^n x_i^2 - \sum_{\ell=1}^m \sum_{i\in B_\ell} (x_i - {\mu_l}_z)^2 \right)\\
    &= \frac{1}{k^2} \left(\text{Var}(X_z) - \text{k-means}(S(X_z)) \right) 
\end{align*}

\end{proof}

Earlier, we showed that for a given clustering, the $k$-means objective of a dataset is the sum of $k$-means objective of the dataset projected along each coordinate. We want to find $S$ such that $\text{k-means}(S(X_{z_\text{min}}))$ is small along $z_\text{min}$, where $z_\text{min}$ is the direction of least variance of $SX$. But, since we do not know $z_\text{min}$,
we find $S$ such that $\text{k-means}(S(X_z))$ is small along every direction $z$. In the previous section, we give instance $k$-means heuristics for this.

\subsection{MIR, Aggregate-level}

Note that in order to minimize the error bound, we must simultaneously minimize the condition number of $(SX)^T SX,$ and the $k$-means objective over the labels $\ty$. Earlier, we justified that $k$-means of the instances X is a good heuristic for both objectives.

\section{Random Bagging Algorithm for Aggregate-MIR}\label{appendix:random}
We propose a random bagging algorithm similar to the one for Bag-LLP (Algorithm \ref{algo:random_bag_eigenvalue}) for Agg-MIR. The upper bound for Agg-MIR (Theorem \ref{thm:MIR-agg-UB1}) is product of the label $k$-means objective and the condition number. We propose the following algorithm which takes both these objectives into account. 

\begin{figure}[!htb]
\begin{mdframed}
\small
\textbf{Input:} : Instances $\mathcal{X}$, fixed bag size $k$, true labels $\Tilde{y}$. \\
\textbf{Steps:}
\begin{enumerate}
    \item Sort points $\mathcal{X}$ in increasing order of $\Tilde{y}$.
    \item Partition sorted points into $m'$ contiguous \emph{super}-bags of sizes $2k$, where $m' = n/2k.$
    \begin{align*}
    \mathcal{X} = \bigcup_{l = 1}^{m'} \mathcal{X}_l \text{  and  } \mathcal{X}_l \bigcap \mathcal{X}_{l'} = \phi \text{  for all } l \neq l'%
    \end{align*}
    \item For $l = 1, \dots, m'$, a $k$-sized bag $B'_l$ is sampled  $u.a.r$ from $\mathcal{X}_l$.
    \item Output $\mathcal{B'}$ where $\mathcal{B'} = \{ B'_l \}_{l \in [m']}$
\end{enumerate}
\end{mdframed}
\caption{Random bagging algorithm for Agg-MIR}\label{algo:random_bag_agg_mir}
\end{figure}

We can analyze the condition number by establishing a lower bound on the minimum eigenvalue of the covariance matrix for the aggregated feature vectors. In Section \ref{random-bagging}, we derive this bound for any fixed partitioning of instances into super-bags, and the same bound holds for Algorithm \ref{algo:random_bag_agg_mir}.

Following the analysis in Section \ref{random-bagging}, we get,
\begin{align*}
\prob \left[ \lambda_{min}\left((S\bX)^TS\bX\right) > (1 - \delta) \frac{\lambda_{min}(X^TX)}{4k^2} \right] 
	\geq 1 - d \cdot \left[ \frac{e^{-\delta}}{(1 - \delta)^{1-\delta}} \right]^{\mu_{\min}/k\beta} 
\label{app_eq:min_eigenvalue_bound_prob_aggmir}
\end{align*}

Let $B_l$ denote a super-bag of size $2k$ for $l \in [m']$. We arbitrarily sample $k$ instances to create a bag $B_l^{(1)}$ and the remaining instances form another bag $B_l^{(2)}$. We know $B_l = B_l^{(1)} \bigcup B_l^{(2)}$ and $B_l^{(1)} \bigcap B_l^{(2)} = \phi$. Also, $|B_l^{(1)}| = |B_l^{(2)}| = k$.

\begin{theorem}
For super-bags $B_l'$ as defined in Algorithm \ref{algo:random_bag_agg_mir} with arbitrary non-overlapping partitions $B_l^{(1)}$ and $B_l^{(2)}$, we have 
\begin{align}
\sum_{l=1}^{m'}\text{k-means-cluster}(\{\Tilde{y}_i\}_{i \in B_l'}) \geq \sum_{l=1}^{m'} \text{k-means-cluster}(\{\Tilde{y}_i\}_{i \in B_l^{(1)}}) + \text{k-means-cluster}(\{\Tilde{y}_i\}_{i \in B_l^{(1)}})
\end{align}
where, k-means-cluster($C$) is the $k$-means clustering loss for cluster $C$. This expands to give the following:
\begin{equation}
\sum_{l=1}^{m'}\sum_{i \in B'_l} (\Tilde{y}_i - \mu_l')^2 \geq \sum_{l=1}^{m'}\big(\sum_{j \in B_l^{(1)}} (\Tilde{y}_i - \mu_l^{(1)})^2 + \sum_{j \in B_l^{(2)}} (\Tilde{y}_i - \mu_l^{(2)})^2 \big)
\end{equation}
where, $\mu$ denotes the respective cluster means.
\end{theorem}
\begin{proof}
    
We write the $k$-means loss for $B'_l$. Let $\mu_l' = \sum_{j \in B'_l}{}\Tilde{y}_i/2k$.
\begin{align*}
    &\sum_{i \in B'_l} (\Tilde{y}_i - \mu_l')^2 \\
    &= \sum_{i \in B'_l} \Tilde{y}_i^2 - 2\Tilde{y}_i\mu_l' + \mu_l'^2 \\
    &= (\sum_{i \in B'_l} \Tilde{y}_i^2) - \frac{(\sum_{i \in B'_l} \Tilde{y}_i)^2}{k} + \frac{(\sum_{i \in B'_l} \Tilde{y}_i)^2}{2k}\\
    &= (\sum_{i \in B'_l} \Tilde{y}_i^2) + (\frac{1}{4k} - \frac{1}{k})(\sum_{i \in B'_l} \Tilde{y}_i)^2 \\
    &= (\sum_{i \in B'_l} \Tilde{y}_i^2) - \frac{1}{2k}(\sum_{i \in B'_l} \Tilde{y}_i)^2 \\
\end{align*}

Next, we write the $k$-means loss for $B_l^{(1)}$. Let $\mu_l^{(1)} = \sum_{j \in B_l^{(1)}}{}\Tilde{y}_i/k$.
\begin{align*}
    &\sum_{j \in B_l^{(1)}} (\Tilde{y}_i - \mu_l^{(1)})^2 \\
    &= \sum_{j \in B_l^{(1)}} \Tilde{y}_i^2 - 2\Tilde{y}_i\mu_l^{(1)} + {\mu_l^{(1)}}^2 \\
    &= (\sum_{j \in B_l^{(1)}} \Tilde{y}_i^2) - \frac{2(\sum_{j \in B_l^{(1)}} \Tilde{y}_i)^2}{k} + \frac{(\sum_{j \in B_l^{(1)}} \Tilde{y}_i)^2}{k}\\
    &= (\sum_{j \in B_l^{(1)}} \Tilde{y}_i^2) - \frac{1}{k} (\sum_{j \in B_l^{(1)}} \Tilde{y}_i)^2 
\end{align*}

Similarly, for $B_l^{(2)}$, we get
\begin{align*}
    \sum_{j \in B_l^{(2)}} (\Tilde{y}_i - \mu_l^{(2)})^2 =
    (\sum_{j \in B_l^{(2)}} \Tilde{y}_i^2) - \frac{1}{k} (\sum_{j \in B_l^{(1)}} \Tilde{y}_i)^2 
\end{align*}

We define $\Delta_l = \sum_{i \in B'_l} (\Tilde{y}_i - \mu_l')^2 - \sum_{j \in B_l^{(1)}} (\Tilde{y}_i - \mu_l^{(1)})^2 - \sum_{j \in B_l^{(2)}} (\Tilde{y}_i - \mu_l^{(2)})^2$.
\begin{align*}
    \Delta_l &= \frac{-1}{2k}  (\sum_{i \in B'_l} \Tilde{y}_i)^2 + \frac{1}{k}  \big[(\sum_{j \in B_l^{(1)}} \Tilde{y}_i)^2 + (\sum_{j \in B_l^{(2)}} \Tilde{y}_i)^2 + 2\sum_{i \in B_l^{(1)}}\sum_{j \in B_l^{(2)}} \Tilde{y}_i \Tilde{y}_j -2\sum_{i \in B_l^{(1)}}\sum_{j \in B_l^{(2)}} \Tilde{y}_i \Tilde{y}_j \big] \\
    &= \frac{-1}{2k}  (\sum_{i \in B'_l} \Tilde{y}_i)^2 + \frac{1}{k}  \big[(\sum_{j \in B_l'} \Tilde{y}_i)^2 -2\sum_{i \in B_l^{(1)}}\sum_{j \in B_l^{(2)}} \Tilde{y}_i \Tilde{y}_j \big] \\
    &= \frac{1}{2k}  (\sum_{i \in B'_l} \Tilde{y}_i)^2 + \frac{-2}{k}  (\sum_{i \in B_l^{(1)}}\sum_{j \in B_l^{(2)}} \Tilde{y}_i \Tilde{y}_j)\\
    &= \frac{1}{2k} \big[ (\sum_{i \in B'_l} \Tilde{y}_i)^2 -4 (\sum_{i \in B_l^{(1)}}\sum_{j \in B_l^{(2)}} \Tilde{y}_i \Tilde{y}_j)\big]\\
    &= \frac{1}{2k} \big[(\sum_{j \in B_l^{(1)}} \Tilde{y}_i) - (\sum_{j \in B_l^{(2)}} \Tilde{y}_i)\big]^2 \\
    &\geq 0
\end{align*}

For any super-bag $B_l'$ for $l \in [m']$, $\Delta_l > 0$. We can now sum over all bags to get the total loss observed after bagging $\Delta = \sum_{l=1}^{m'}\Delta \geq 0$.

This implies that the loss incurred by applying the $k$-means objective is higher when the instances are clustered into super-bags of sizes $2k$, compared to our random bagging approach, which creates two non-overlapping bags of sizes $k$ from the super-bags.

\end{proof}

\section{Analysis for GLMs}\label{GLM}

We generalize the previous results for linear regression to the setting of Generalized Linear Model's (GLMs), which includes popular paradigms such as logistic regression. We study both instance-level and aggregate-level losses for MIR under the GLM framework. For instance-MIR, we derive an upper bound that leads to label k-means clustering as the optimal bagging strategy. This result holds across all distributions within the exponential family. For aggregate-MIR, our objective suggests minimizing the range between the maximum and minimum expected instance labels within a bag, implying that features with similar expected labels should be grouped together, yielding a clustering-based outcome. This result holds for exponential distributions which have a monotonic first derivative. The detailed analysis is provided below.

\sushant{copied above overview from main paper, move lemma to proofs}

\subsection{MIR}
We now generalize our derivation to the class of \emph{generalized linear models} (GLMs). The instance-level labels $y_i$ are conditionally independent given $\bx_i$ in GLMs, and are drawn from a specific distribution within the exponential family. The corresponding log-likelihood function can be expressed as:
\begin{align}\label{eq:GLM}
    \log p(y_i \mid \eta_i,\phi) = \frac{y_i\eta_i -b(\eta_i)}{a_i(\phi)} + c(y_i,\phi)\,,
\end{align}
where $\eta_i$ is a location variable and $\phi$ is the scaling variable.
The functions $a_i$, $b$, and $c$ are provided.
We can take $a_i(\phi) = \phi/w_i$, where $w_i$ is a constant prior information.
We analyse canonical GLMs, in which $\eta_i = \bx_i^T \tth$ for an unknown model ~$\tth$. Some properties of GLMs are $\mu = \E[y|x] = b'(x^T\tth)$ and $Var(y|x) = a(\phi)b''(x^T\tth)$. We consider $\mathcal{L}$ to the negative log likelihood and we can ignore the term $c(y_i,\phi)$ as it does not depend on $\theta$. Our objective is to find a bagging strategy which closes the gap between the true model $\tth$ and $\hth$. For GLMs we achieve this by minimizing the gradient of the loss at $\tth$. 

\subsubsection{Analysis of instance-level loss for MIR}

\begin{lemma}\label{lem:min_grad_strongly_convex}
Suppose that the loss $\mathcal{L}$ is strongly convex with parameter $\mu$ and $\hth = \arg\min_{\theta}\mathcal{L}(\theta)$. Then, for any model $\tth$, we have
\[
\|\hth - \tth\|_2 \le \frac{1}{\mu}\|\mathcal{L}(\tth)\|_2.
\]
In addition, if $\mathcal{L}$ has a Lipschitz continuous gradient with parameter $L$, we have
\[
 \frac{1}{L}\|\mathcal{L}(\tth)\|_2 \le \|\hth - \tth\|_2.
\]
\end{lemma}

Let $\hth$ be the minimizer of the instance-level loss:
  \begin{align}\label{eq:GLM_estimator}
      \hth &= \argmin_{\theta} \frac{1}{n} \sum_{l=1}^m \sum_{i \in B_l}\frac{\overline{y_l}\eta_i -b(\eta_i)}{a_i(\phi)}
  \end{align}
 
 We find the optimal $\hth$ by solving $\nabla\mathcal{L} (\hth) = \mathbf{0}$. 
 We use Lemma \ref{lem:min_grad_strongly_convex} which states that $\|\hth - \tth\|_2$ is lower bounded by $\|\nabla \mathcal{L}(\tth)\|_2$ for strongly convex functions. 
 
 We define a instance-level attribution matrix for MIR, $A$ $\in \{0, 1\}^{n \times n}$, which assigns the bag label to each feature vector in the bag. The prime feature vector is chosen uniformly at random.
 Let $\overline{y} = [\overline{y_1}, \dots, \overline{y_m}]$, where $\overline{y_l} = y(\Gamma(B_l))$ as previously defined.
\begin{align}\label{eq:Avent}
    {A}_{(i,j)} =
    \begin{cases}
        1 & \text{if } i \in B_l \text{ and } \overline{y_l} = y(x_j)\\
        0 & \text{otherwise}.
    \end{cases}
\end{align}

We define $S$ $\in [0, 1]^{n \times n}$ as the expectation of $A$:
\begin{align}
    {S}_{(i,j)} =
    \begin{cases}
        \frac{1}{|B_l|} & \text{if } i, j \in B_l \\
        0 & \text{otherwise}.
    \end{cases}
\end{align}

\begin{theorem}
If we consider canonical GLMs with $\eta_i = \bx_i^T\tth$, then we have
\begin{align}
    \E\left[\|\nabla \mathcal{L}(\tth)\|_2 \right]
    &\leq m(\|b'(X\tth)\|_2^2 + \|Db''(X\tth)\|_1) + \|(S - I) b'(X\tth)\|_2^2 - \|Sb'(X\tth)\|_2^2 
\end{align}
  where, $D = \text{Diag}(\{a_i(\phi)\})$.
\label{thm:glm-mir-event-loss}
\end{theorem}
\begin{proof}
    We begin by computing $\nabla\mathcal{L} (\theta)$ and expressing it in the matrix format:
    \begin{align*}
        \nabla\mathcal{L} (\theta) &= \frac{1}{n} \sum_{l = 1}^m \sum_{i \in B_l} \frac{(\overline{y_l} - b'(x_i^T\theta))x_i}{a_i(\phi)} \\
        &= X^TD^{-1}(Ay - b'(X\theta))
    \end{align*}
    where, $D := \text{Diag}(\{a_i(\phi)\})$.
    \begin{align*}
        \E\left[\|\nabla\mathcal{L} (\theta)\|_2^2 | X\right] &= \E\left[\|X^TD^{-1}(Ay - b'(X\theta))\|_2^2 | X\right] \\
        &\leq \|X^TD^{-1}\|_{op}^2 \E\left[\|Ay - b'(X\theta)\|_2^2 | X\right] \\
        &= \|X^TD^{-1}\|_{op}^2 \E\left[(Ay - b'(X\theta))^T(Ay - b'(X\theta)) | X\right] \\
        &= \|X^TD^{-1}\|_{op}^2 \E\left[(Ay)^T(Ay) - b'(X\theta)^TAy - (Ay)^Tb'(X\theta) + b'(X\theta)^T b'(X\theta)| X\right] \\
        &= \|X^TD^{-1}\|_{op}^2 \left(\E\left[(Ay)^T(Ay)| X\right] - b'(X\theta)^TSy - (Sy)^Tb'(X\theta) + b'(X\theta)^T b'(X\theta) \right) \\
        &= \|X^TD^{-1}\|_{op}^2 \big(\E\left[(Ay)^T(Ay)| X\right] - b'(X\theta)^T Sy - (Sy)^T b'(X\theta) + b'(X\theta)^T b'(X\theta)\\
        & \qquad + (Sb'(X\theta))^T (Sb'(X\theta)) - (Sb'(X\theta))^T (Sb'(X\theta)) \big) \\
        &= \|X^TD^{-1}\|_{op}^2 \left(\E\left[\|Ay\|_2^2| X\right] + \|(S - I)b'(X\theta)\|_2^2 - \|Sb'(X\theta)\|_2^2 \right) \\
        &\leq \|X^TD^{-1}\|_{op}^2 \left(\E\left[\|A\|_{op}^2\|y\|_2^2| X\right] + \|(S - I)b'(X\theta)\|_2^2 - \|Sb'(X\theta)\|_2^2 \right) \\
        &\leq \|X^TD^{-1}\|_{op}^2 \left(m(\|b'(X\tth)\|_2^2 + \|Db''(X\tth)\|_1) + \|(S - I)b'(X\theta)\|_2^2 - \|Sb'(X\theta)\|_2^2 \right)
    \end{align*}
\end{proof}

Note that the term $\|X^TD^{-1}\|_{op}^2$ is constant and the first term $m(\|b'(X\tth)\|_2^2 + \|Db''(X\tth)\|_1)$ is independent of the bagging strategy, it can be disregarded. Thus, we focus on the remaining terms to derive a clustering objective for event-level MIR. To proceed, we expand the matrix notation and express these terms as a summation over instances. We define $\mu_l := \frac{\mu_i}{|B_l|}$, where $\mu_i = \E[y_i | x_i] = b'(x_i^T\tth)$.

\begin{align*}
    \min_{(B_1,\dots,B_m) \in \mathcal{B}} \quad \|(S - I)b'(X\theta)\|_2^2 - \|Sb'(X\theta)\|_2^2 &= \min_{(B_1,\dots,B_m) \in \mathcal{B}} \quad \sum_{l=1}^m \sum_{i \in B_l} (\mu_i - \mu_l)^2 - \sum_{l=1}^m |B_l|\mu_l
\end{align*}
Minimizing the first term in the objective is similar to performing $1d$ $k$-means clustering and maximizing the second term forces the bags to be of larger sizes.

\subsubsection{Analysis of Aggregate MIR loss}
Let $\hth$ be the minimizer of the aggregate MIR loss:
  \begin{align}\label{eq:GLM_estimator_agg_mir}
      \hth &= \argmin_{\theta} \frac{1}{m} \sum_{l=1}^m \frac{\overline{y_l}\sum_{i \in B_l}\frac{\eta_i}{|B_l|} -b(\sum_{i \in B_l}\frac{\eta_i}{|B_l|})}{a_l(\phi)}
  \end{align}
 
 The steps involved in analysing this function are similar to the instance-loss function. We find the optimal $\hth$ by solving $\nabla\mathcal{L} (\hth) = \mathbf{0}$ and then minimize $\|\nabla \mathcal{L}(\tth)\|_2$ to approximate $\|\hth - \tth\|_2$ (Lemma \ref{lem:min_grad_strongly_convex}).

\begin{theorem}
If we consider canonical GLMs with $\eta_i = \bx_i^T\tth$, then we get
\begin{align}
    \E\left[\|\nabla \mathcal{L}(\tth)\|_2 \right]
    &\leq n \lambda_{max}(X^TX)\left(m(\|b'(X\tth)\|_2^2 + \|Db''(X\tth)\|_1) + \|Sb'(X\theta) - b'(SX\theta)\|_2^2 - \|Sb'(X\theta)\|_2^2 \right))
\end{align}
  where, $D = \text{Diag}(\{a_i(\phi)\})$.
\label{thm:glm-agg-mir}
\end{theorem}
\begin{proof}
    We begin by computing $\nabla\mathcal{L} (\theta)$ and expressing it in the matrix format:
    \begin{align*}
        \nabla\mathcal{L} (\theta) &= \frac{1}{n} \sum_{l = 1}^m  \frac{(\overline{y_l} - b'(\sum_{i \in B_l}\frac{x_i^T\theta}{|B_l|}))\sum_{i \in B_l}\frac{x_i^T\theta}{|B_l|}}{a_l(\phi)} \\
        &= (SX)^TD^{-1}(Ay - b'(SX\theta))
    \end{align*}
    where, $D := \text{Diag}(\{a_l(\phi)\})$.
    \begin{align*}
        \E\left[\|\nabla\mathcal{L} (\theta)\|_2^2 | X\right] &= \E\left[\|(SX)^TD^{-1}(Ay - b'(SX\theta))\|_2^2 | X\right] \\
        &\leq \|(SX)^TD^{-1}\|_{op}^2 \E\left[\|Ay - b'(SX\theta)\|_2^2 | X\right] \\
        &= \|(SX)^TD^{-1}\|_{op}^2 \E\left[(Ay - b'(SX\theta))^T(Ay - b'(SX\theta)) | X\right] \\
        &= \|(SX)^TD^{-1}\|_{op}^2 \E\left[(Ay)^T(Ay) - b'(SX\theta)^TAy - (Ay)^Tb'(SX\theta) + b'(SX\theta)^T b'(SX\theta)| X\right] \\
        &= \|(SX)^TD^{-1}\|_{op}^2 \left(\E\left[(Ay)^T(Ay)| X\right] - b'(SX\theta)^TSy - (Sy)^Tb'(SX\theta) + b'(SX\theta)^T b'(SX\theta) \right) \\
        &= \|(SX)^TD^{-1}\|_{op}^2 \big(\E\left[(Ay)^T(Ay)| X\right] - b'(SX\theta)^TSb'(X\theta) - (Sb'(X\theta))^Tb'(SX\theta) +  \\ 
        & \ \ \ \ b'(SX\theta)^T b'(SX\theta) + (Sb'(X\theta))^T (Sb'(X\theta)) - (Sb'(X\theta))^T (Sb'(X\theta)) \big) \\
        &= \|(SX)^TD^{-1}\|_{op}^2 \left(\E\left[\|Ay\|_2^2| X\right] + \|Sb'(X\theta) - b'(SX\theta)\|_2^2 - \|Sb'(X\theta)\|_2^2 \right) \\
        &\leq \|(SX)^TD^{-1}\|_{op}^2 \left(\E\left[\|A\|_{op}^2\|y\|_2^2| X\right] \|Sb'(X\theta) - b'(SX\theta)\|_2^2 - \|Sb'(X\theta)\|_2^2 \right) \\
        &\leq \|(SX)^TD^{-1}\|_{op}^2 \left(m(\|b'(X\tth)\|_2^2 + \|Db''(X\tth)\|_1) + \|Sb'(X\theta) - b'(SX\theta)\|_2^2 - \|Sb'(X\theta)\|_2^2 \right)\\
        &\leq \|D^{-1}\|_{op}^2 \lambda_{max}(X^TX)\left(m(\|b'(X\tth)\|_2^2 + \|Db''(X\tth)\|_1) + \|Sb'(X\theta) - b'(SX\theta)\|_2^2 - \|Sb'(X\theta)\|_2^2 \right)
    \end{align*}
\end{proof}
We now show how the final objective in Equation \ref{thm:glm-agg-mir} leads to a clustering objective. The key term in this objective which depends on $S$ is $\|Sb'(X\theta) - b'(SX\theta)\|_2^2$. Our task is to determine the optimal bagging matrix $S$ that would minimize this term. To simplify this expression and develop an interpretable algorithm, we assume that the function $b'(.)$ is monotonic.
Focusing on the case where $b'(.)$ is an increasing function, we know that $b'(t_1) \geq b'(t_2)$ $\iff$ $t_1 \geq t_2$. 
\begin{align*}
    \Big\|(Sb'(X\theta) - b'(SX\theta)\Big\|_2^2 &=
    \sum_{l=1}^m \left(\sum_{x \in B_l} \frac{b'(x^T\tth)}{|B_l|}  -  b'(\sum_{x \in B_l} \frac{x^T\tth}{|B_l|})\right)^2
\end{align*}
Since $b'$ is an increasing function, the inequality $b'(\max_{x' \in B_l} x'^T\tth) \geq b'(x^T\tth)$ holds true for all $x \in B_l$ (and $\max_{x' \in B_l} x'^T\tth \geq x^T\tth$). Similarly, $b'(x^T\tth) \geq b'(\min_{x' \in B_l} x'^T\tth)$ would hold true for all $x \in B_l$ $x^T\tth \geq \min_{x' \in B_l})$ We now look at the first term:
\begin{align*}
    \frac{b'(\min_{x' \in B_l} x'^T\tth)}{|B_l|} &\leq \sum_{x \in B_l}  \frac{b'(x^T\tth)}{|B_l|} \leq  \frac{b'(\sum_{x \in B_l} \max_{x' \in B_l} x'^T\tth)}{|B_l|} \\
    b'(\min_{x' \in B_l} x'^T\tth) &\leq \sum_{x \in B_l}  \frac{b'(x^T\tth)}{|B_l|} \leq b'(\max_{x' \in B_l} x'^T\tth)
\end{align*}

We bound the second term:
\begin{align*}
    b'( \sum_{x \in B_l} \frac{\min_{x' \in B_l} x'^T\tth}{|B_l|}) &\leq b'(\sum_{x \in B_l} \frac{x^T\tth}{|B_l|}) \leq  b'(\frac{\sum_{x \in B_l} \max_{x'} x'T\tth}{|B_l|}) \\
    b'(\min_{x' \in B_l} x'^T\tth) &\leq b'(\sum_{x \in B_l} \frac{x^T\tth}{|B_l|}) \leq b'(\max_{x' \in B_l} x'^T\tth)
\end{align*}

It is easy to see that the difference $\|Sb'(X\theta) - b'(SX\theta)\|_2^2$ has an upper bound:
\begin{align}
    \sum_{l=1}^m \left(\sum_{x \in B_l} \frac{b'(x^T\tth)}{|B_l|}  -  b'(\sum_{x \in B_l} \frac{x^T\tth}{|B_l|})\right)^2 &\leq
    \sum_{l=1}^m \left(b'(\max_{x' \in B_l} x'^T\tth) -  b'(\min_{x' \in B_l} x'^T\tth)\right)^2
    \label{eq:glm_agg_mir_clustering}
\end{align}

If $n = mk$ and we need to construct-equal sized bags having k instances each, then the minimization of Equation \ref{eq:glm_agg_mir_clustering} can be achieved by sorting $b'(x^T\tth)$ for all $x \in X$, and dividing the points into contiguous chunks of size $k$. This process resembles the $1d$ clustering objective with an equal-size constraint.

The monotonicity condition holds true for majority of the distributions belonging to the exponential family including normal, poisson, logistic and inverse gaussian.

\end{document}